%% file: main.tex
\makeatletter
\expandafter\def\csname ver@everyshi.sty\endcsname{3000/12/31}
\makeatother
\PassOptionsToPackage{dvipsnames, table}{xcolor}
\documentclass[10pt,twocolumn,letterpaper]{article}

\usepackage{cvpr}              %

\input{preamble}

\definecolor{cvprblue}{rgb}{0.21,0.49,0.74}
\usepackage[pagebackref,breaklinks,colorlinks,allcolors=cvprblue]{hyperref}

\def\paperID{*****}
\def\confName{CVPR}
\def\confYear{2026}

\title{\methodname: Rotation-Invariant Non-Rigid Correspondences}

\author{Maolin Gao$^{1,3}$ ~~~~
Shao Jie Hu-Chen$^{1,3}$ ~~~~
Congyue Deng$^{2,4}$ \\ 
Riccardo Marin$^{1,3}$ ~~~~
Leonidas Guibas$^{2}$  ~~~~
Daniel Cremers$^{1,3}$ \\
$^1$ TUM
\quad
$^2$ Stanford University
\quad 
$^3$ MCML
\quad 
$^4$ MIT
}
\begin{document}

\twocolumn[{%
\renewcommand\twocolumn[1][]{#1}%
\maketitle
\begin{center} %
\vspace{-1.0cm}
    \captionsetup{type=figure}
    
    \makebox[\textwidth][c]{%
        \hspace{-0.0cm}%
        \includegraphics[width=2.4\columnwidth,trim=400 720 450 630, clip]{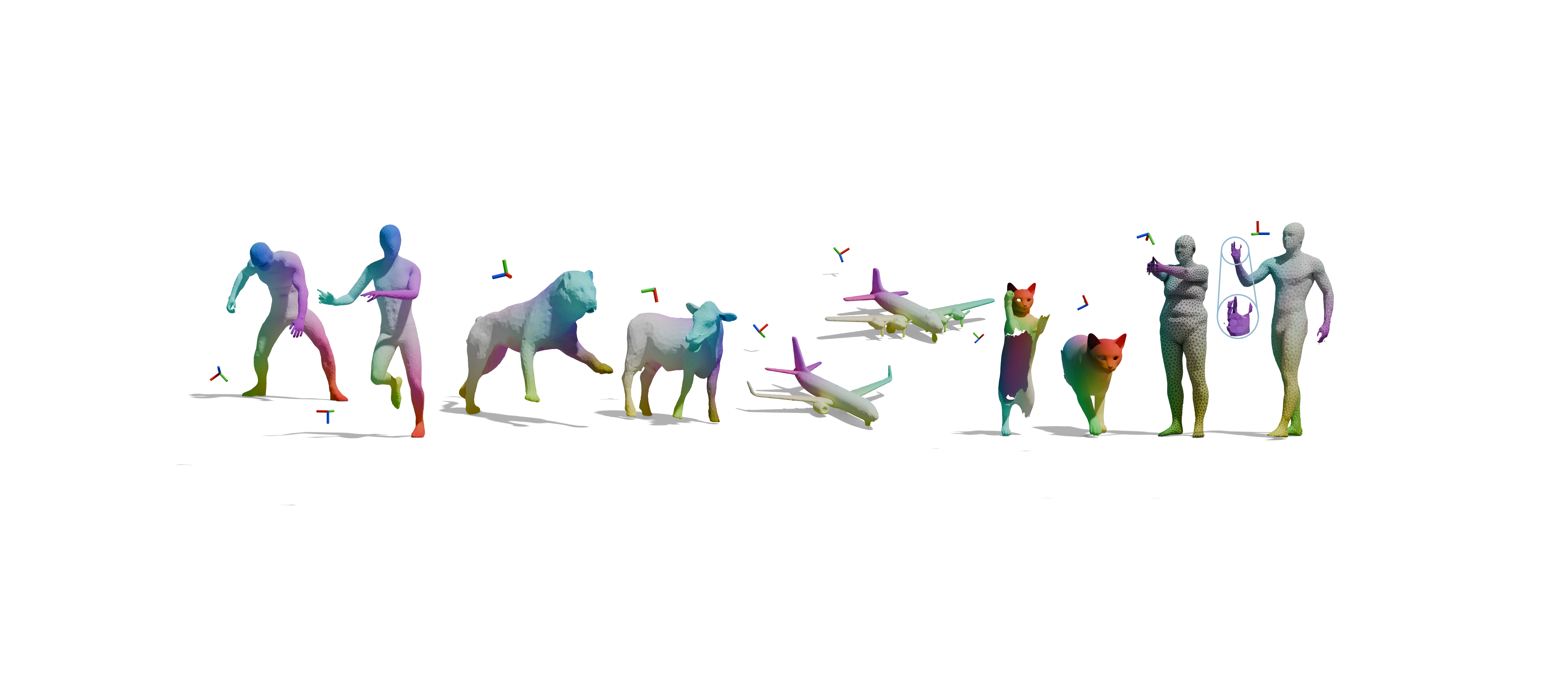}%
    }
    
    \caption{
    We propose \textbf{\methodname}, an unsupervised framework unifying rigid and non-rigid shape matching.
    Our end-to-end point feature extractor, the \textit{\netname}, is invariant to extrinsic rotations and robust to intrinsic symmetries, without relying on handcrafted descriptors (shape orientations indicated by RGB frames).
    It handles challenging non-rigid matching cases with arbitrary poses, non-isometry, partiality, non-manifoldness, and noise. From left to right we show pairs of humanoid, quadruped, airplane, partial and non-manifold shapes.
    }
    \label{fig:reb_teaser}
    
\vspace{0.2cm}
\end{center}%
}]

\begin{abstract}

Dense 3D shape correspondence remains a central challenge in computer vision and graphics as many deep learning approaches still rely on intermediate geometric features or handcrafted descriptors,  limiting their effectiveness under non-isometric deformations, partial data, and non-manifold inputs. To overcome these issues, we introduce \methodname, an unsupervised, rotation-invariant dense correspondence framework that effectively unifies rigid and non-rigid shape matching. The core of our method is the novel \netname, a feature extractor that integrates vector-based SO(3)-invariant learning with orientation-aware complex functional maps to extract robust features directly from raw geometry. This allows for a fully end-to-end, data-driven approach that bypasses the need for shape pre-alignment or handcrafted features. Extensive experiments show unprecedented performance of \methodname~across challenging non-rigid matching tasks, including arbitrary poses, non-isometry, partiality, non-manifoldness, and noise.

\end{abstract}

\input{sec/1_intro}

\input{sec/2_related_work_compact}

\input{sec/2_background}

\input{sec/3_network}

\input{sec/4_method}

\input{sec/5_experiment}

\input{sec/6_conclusion}

{
    \small
    \bibliographystyle{ieeenat_fullname}
    \bibliography{main}
}

\input{sec/X_suppl}

\end{document}

%% file: preamble.tex
\usepackage[accsupp]{axessibility}  %

\newcommand{\duofm}{DUOFM}
\newcommand{\consistfm}{CnsFM}

\newcommand{\ulrssm}{URSSM}

\newcommand{\sfm}{SmpFM}
\newcommand{\hybridfm}{HbrFM}
\newcommand{\sms}{SMS}
\newcommand{\echomatch}{EchoMatch}
\newcommand{\wormhole}{Wormhole}
\newcommand{\diffusionnet}{DiffusionNet}

\input{command_notation}

\usepackage{graphicx}
\usepackage{tabularx}
\usepackage{amsmath}
\usepackage{amsthm}
\newtheorem{theorem}{Theorem}%

\usepackage{mathtools}
\usepackage{changepage}
\usepackage{multirow}

\usepackage{soul} %
\usepackage[normalem]{ulem} %
\usepackage{adjustbox}

\usepackage[]{overpic}
\usepackage{stfloats}

\newcommand{\methodname}{RINO}
\newcommand{\netname}{RINONet}
\newcommand{\wks}{wks}
\newcommand{\xyz}{xyz}
\newcommand{\hks}{hks}
\newcommand{\shot}{shot}

\usepackage{pgfplots}
\usepgfplotslibrary{statistics}
\pgfplotsset{compat=1.17} 
\usepackage{pgf,tikz}

\definecolor{cPLOT1}{RGB}{80,150,80}
\definecolor{cPLOT3}{RGB}{247,179,43}

%% file: command_notation.tex
\DeclareMathOperator*{\argmin}{arg\,min}

\newcommand{\vb}{\mathbf{b}}
\newcommand{\vd}{\mathbf{d}}
\newcommand{\ve}{\mathbf{e}}

\newcommand{\vf}{\mathbf{f}}
\newcommand{\vg}{\mathbf{g}}
\newcommand{\vh}{\mathbf{h}}

\newcommand{\vw}{\mathbf{w}}
\newcommand{\vu}{\mathbf{u}}

\newcommand{\vz}{\mathbf{z}}

\newcommand{\mA}{\mathbf{A}}
\newcommand{\mC}{\mathbf{C}}

\newcommand{\mX}{\mathbf{X}}

\newcommand{\mV}{\mathbf{V}}
\newcommand{\mF}{\mathbf{F}}

\newcommand{\mR}{\mathbf{R}}
\newcommand{\mQ}{\mathbf{Q}}

\newcommand{\mI}{\mathbf{I}}

\newcommand{\mM}{\mathbf{M}}

\newcommand{\mPhi}{\mathbf{\Phi}}

\newcommand{\mPi}{\mathbf{\Pi}}

\newcommand{\cH}{\mathcal{H}}
\newcommand{\cG}{\mathcal{G}}
\newcommand{\cF}{\mathcal{F}}
\newcommand{\cT}{\mathcal{T}}
\newcommand{\cX}{\mathcal{X}}
\newcommand{\cY}{\mathcal{Y}}

\newcommand{\cL}{\mathcal{L}}

\newcommand{\rmSO}[1]{\mathrm{S}\mathrm{O}(#1)}

\newcommand{\nx}{n_{\cX}}
\newcommand{\ny}{n_{\cY}}

\newcommand{\mCxy}{\mC_{\cX\cY}}
\newcommand{\mCyx}{\mC_{\cY\cX}}

\newcommand{\mQxy}{\mQ_{\cX\cY}}
\newcommand{\mQyx}{\mQ_{\cY\cX}}
\newcommand{\mPhix}{\mPhi_{\cX}}
\newcommand{\mPhiy}{\mPhi_{\cY}}

\newcommand{\mPixx}{\mPi_{\cX\cX}}

\newcommand{\mPixy}{\mPi_{\cX\cY}}
\newcommand{\mPiyx}{\mPi_{\cY\cX}}
\newcommand{\mFx}{\mF_{\cX}}
\newcommand{\mFy}{\mF_{\cY}}

\newcommand{\mChxy}{\hat{\mC}_{\cX\cY}}
\newcommand{\mQhxy}{\hat{\mQ}_{\cX\cY}}
\newcommand{\mChxx}{\hat{\mC}_{\cX\cX}}
\newcommand{\mChyy}{\hat{\mC}_{\cY\cY}}

\newcommand{\diag}{\mathrm{diag}}
\newcommand{\re}{\mathrm{Re}}
\newcommand{\Edata}{\mathcal{E}_\mathrm{data}}
\newcommand{\Ereg}{\mathcal{E}_\mathrm{reg}}
\newcommand{\relu}{\mathrm{ReLU}}

\newcommand*{\VNINV}{\mathrm{VN}\textsf{-}\mathrm{invariant}}
\newcommand*{\VNLIN}{\mathrm{VN}\textsf{-}\mathrm{linear}}
\newcommand*{\VNMLP}{\mathrm{VN}\textsf{-}\mathrm{MLP}}
\newcommand{\Ltotal}{L_{\mathrm{total}}}
\newcommand{\Lbij}{L_{\mathrm{bij}}}
\newcommand{\Lorth}{L_{\mathrm{orth}}}
\newcommand{\Lstruct}{L_{\mathrm{struct}}}
\newcommand{\Lcouple}{L_{\mathrm{couple}}}
\newcommand{\Lcontrast}{L_{\mathrm{contr}}}
\newcommand{\Lpc}{L_{\mPi\mC}}
\newcommand{\Lpq}{L_{\mPi\mQ}}

\newcommand{\bbR}{\mathbb{R}}
\newcommand{\bbC}{\mathbb{C}}

%% file: sec/1_intro.tex
\vspace{-5pt}
\section{Introduction}
\label{sec:intro}

Establishing correspondence between 3D shapes is a foundational and challenging problem in computer vision and graphics.
While deep learning has enhanced classical shape matching by replacing handcrafted descriptors with learned features, most approaches still depend on intermediate geometric representations rather than operating directly on raw shape data \cite{litany2017deep,cao2023unsupervised,halimi2019unsupervised,attaiki2023clover, magnet2024memory, donati2022DeepCFMaps}.
This reliance on human-designed features, which are often specialized (e.g., for isometric shapes), fundamentally limits robustness and generalization across diverse, challenging conditions like non-isometric deformations, partiality, and non-manifold structures. 
Hence, a compelling alternative is to start learning directly from the raw data — namely, using the vertex coordinates as input. Such an approach appears appealing for several reasons: it does not require preprocessing, it does not rely on theoretical assumptions, and it has been the most successful approach in other domains (e.g., CNNs generally take RGB pixel values as input directly, LLMs start from words or sub-words tokens). In accordance with Sutton's bitter lesson \cite{Sutton2019BitterLesson}, we believe that the path toward robust shape correspondence lies in scalable end-to-end learning, with little or no human-encoded sophistication.

However, directly processing raw 3D geometry introduces the challenge of \textit{shape–pose entanglement}, which has traditionally led to splitting the problem into two separate subproblems: rigid matching, estimating global rotations and translations; and non-rigid matching, assuming pose alignment to infer dense correspondences.
This separation, however, is inherently ill-posed, as global transformations and local deformations are tightly coupled.
The distinction becomes especially ambiguous when canonical poses are hard to define, for example, between a running and a sitting human, or in cases of partial shape matching.

To eliminate the need for rigid pose alignment, fully intrinsic shape matching methods~\cite{donati2022DeepCFMaps, cao2023unsupervised} leverage handcrafted intrinsic features (e.g., $\text{\wks}$~\cite{aubry_wave_2011}) and intrinsic architectures (e.g., $\text{\diffusionnet}$~\cite{sharp2022diffusionnet}).
Yet, reliance solely on intrinsic information fundamentally limits their scope: it cannot theoretically distinguish shape symmetries, and it yields severely degraded and incompatible features for non-smooth or irregular shapes~\cite{liu2017dirac, ye2018dirac, ye2020dissertation, gao2023sigma, gao2021multi}.

In this work, we propose a paradigm that goes beyond the need of manual pose alignment and does not suffer from any of the previous limitations. 
Our insight stems from three properties which we believe coexist in the ideal shape descriptor: first, it should be robust to noise and artifacts; second, it incorporates the geometric prior provided by the underlying surface; finally, it characterizes shape vertices regardless of the extrinsic orientation of the shape. To our knowledge, no available methods respect all these points.
To this end, our idea is to develop a rotation-invariant architecture, which factors out the global shape orientation from the articulations' non-rigid deformation. 
In particular, we revise the components of the popular \diffusionnet~\cite{sharp2022diffusionnet} to be SO(3)-equivariant end-to-end, using a vector neuron representation. Such modifications are well-motivated by theoretical derivations that we report in the appendix. 
Our network can learn SO(3)-invariant, geometric features directly from raw 3D geometry. The feature gradients are further integrated with the orientation-aware Complex Functional Maps (CFMaps) \cite{donati2022CFMaps, donati2022DeepCFMaps} to map tangent bundles between shapes. 
The entire system is then trained using a novel unsupervised loss that unifies Functional Maps (FMaps), CFMaps, and pointwise maps, creating a unified framework that leverages their complementary strengths for achieving high-quality, robust features and correspondences (Fig.~\ref{fig:reb_teaser}-\ref{fig:SO3-inv}).
To summarize, our key contributions are:

\begin{itemize}     
\item We introduce the first unsupervised rotation-invariant dense correspondence method, \textbf{\methodname}, that unifies rigid and non-rigid matching.
\item Our novel feature extractor, \textbf{\netname}, achieves direct, end-to-end, and robust feature extraction from raw geometry, and discovers complex, scalable features without the need for handcrafted descriptors or shape pre-alignment.
\item We integrate SO(3)-invariant learning with orientation-aware complex FMaps on tangent vector fields, yielding a coherent network architecture and unsupervised loss. 
\item Extensive experiments demonstrate superior performance across challenging non-rigid matching scenarios, including arbitrary poses, non-isometric deformations, partiality, non-manifold structures, and noise perturbations.
\end{itemize}

%% file: sec/2_related_work_compact.tex
\vspace{-3pt}
\section{Related Works}
\label{sec:related_works}

Non-rigid shape matching aims to find correspondences between points on two shapes. We focus on the most relevant works below and refer to \cite{van2011survey, sahillioglu_recent_2020} for an extensive review.

\vspace{-5pt}
\paragraph{Learning Non-Rigid Correspondences.}
The Deep Functional Maps framework \cite{ovsjanikov2012functional, litany2017deep} enabled deep learning in non-rigid matching. It uses a learnable feature extractor to refine initial descriptors (e.g., \shot~\cite{salti2014shot}) before computing a functional map. This framework has since been extended to zero-shot \cite{attaiki2022ncp, attaiki2023snk}, weakly supervised \cite{sharma2020weakly} and unsupervised settings \cite{halimi2019unsupervised, cao2022unsupervised}, incorporating more advanced objectives \cite{donati2020deepGeoMaps, sun2023spatially, le2024optimal}. While some explored learning bases \cite{marin2020lie, jiang2023nie, gao2025coe}, the theoretical strength of the Laplace-Beltrami Operator (LBO) bases remains preferred.
As a feature extractor, initial works used ResNet with \shot~or PointNet++/KPConv with \xyz~\cite{sharma2020weakly, donati2020deepGeoMaps},  but lately \diffusionnet~\cite{sharp2022diffusionnet} has emerged as the standard for learning robust descriptors \cite{attaiki2021dpfm, cao2022unsupervised, donati2022DeepCFMaps, sun2023spatially, attaiki2023clover} among other networks~\cite{wiersma2022deltaconv, maesumi2025poissonnet}. Notably, the focus of these works has been mainly on designing losses. We instead propose a novel feature extractor that learns SO(3)-invariant features directly from raw 3D geometry.

\begin{figure}[t!] 
\vspace{-0.7cm}
    \centering
    \includegraphics[width=0.9\columnwidth,trim=150 265 270 20, clip]{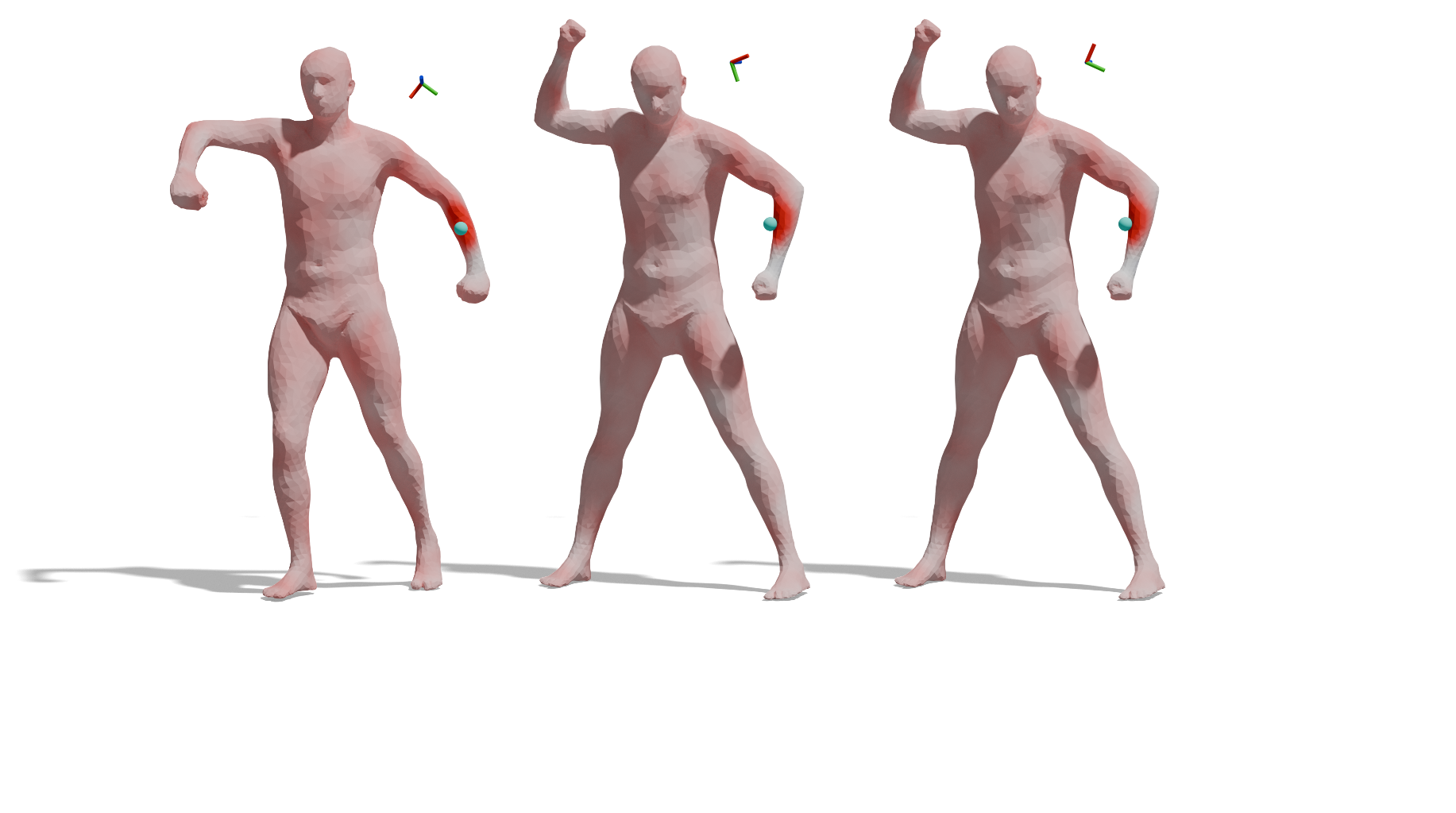}
    \caption{\textbf{Our learned SO(3)-invariant features.} We visualize the Euclidean distance to the blue surface point in the learned feature space (darker red means higher similarity). Semantic correspondences have a similar feature similarity pattern (left \& mid.). The learned feature is invariant to rotations (mid. \& right). Shape orientations are depicted by RGB frames.}
    \label{fig:feat}
\vspace{-0.3cm}
\end{figure}

\vspace{-5pt}
\paragraph{Shape Descriptors in 3D Matching.}

Historically, 3D shape analysis relied on handcrafted descriptors like \hks~\cite{sun_concise_2009}, \wks~\cite{aubry_wave_2011}, and \shot~\cite{salti2014shot} to establish canonical feature spaces, which modern deep learning models now use as refined inputs. However, \shot~is highly sensitive to discretization and mesh connectivity~\cite{attaiki2021dpfm}, while \wks~is limited by an isometric assumption (Sec.~\ref{sec:exp}) that fails under significant non-rigid deformation and cannot theoretically disambiguate intrinsic symmetry. This has led to an unprincipled reliance on pre-aligned \xyz~coordinates in challenging datasets~\cite{cao2023unsupervised, magnet2024memory, pierson2025diffumatch, huang2013consistent, bastian2023hybrid}. While recent foundational 2D features (e.g., DINOv3) offer high expressivity~\cite{dutt2024diff3f, chen2025dv, xiehm2025echomatch}, they lack geometric understanding of shape orientation, require upright alignment for multi-view aggregation, and struggle to resolve rotational equivariance from data alone. To address these gaps, we propose a principled SO(3)-invariant feature extractor that directly processes raw \xyz~coordinates, unifying the input feature choice and enabling genuine data-driven learning from 3D shapes.

\vspace{-5pt}
\paragraph{Intrinsic Shape Symmetry.}
Current approaches either use non-rigorous orientation-aware features \cite{ren2018continuous, gao2023sigma} or rely on impractical extrinsic pre-alignment \cite{sharma2020weakly, pierson2025diffumatch} to resolve symmetries. While \diffusionnet~combined with intrinsic descriptors attempts to learn anisotropic diffusion, it assumes consistently oriented normals, which rarely applies in real data.
This leads to a lack of performance guarantees, especially for partial shapes. The use of SO(3)-\emph{variant} \diffusionnet~with \xyz~further hinders symmetry disambiguation due to poor handling of shape rotations (cf. Sec.~\ref{sec:sym}). We build upon CFMaps \cite{donati2022CFMaps}, which theoretically encodes only orientation-preserving maps, and combine it with our learned SO(3)-invariant features from arbitrarily oriented raw geometry. This eliminates intrinsic symmetry flips without requiring extrinsic pre-alignment.

\vspace{-5pt}
\paragraph{Equivariant Learning in 3D Understanding.}
Equivariant neural networks (ENN) inject group symmetry into the learning process \cite{cohen2016gcnn, bronstein2017geometricDL,dehaan2021, poulenard2019effective}. Architectures like Equivariant Point Network \cite{chen2021equivariant} offer approximate SE(3)-equivariance, while Tensor Field Networks \cite{thomas2018tensor} and e3nn \cite{geiger2022e3nn} achieve strict continuous SO(3)-equivariance using steerable convolutions, often at a high computational cost. While successful in rigid tasks \cite{zhong2023multi, Park_2024_CVPR}, applying SO(3)-equivariant networks to fine-grained, non-rigid correspondence remains a less-explored challenge.

%% file: sec/2_background.tex
\section{Background}

\paragraph{Functional Maps \& Complex Functional Maps.}
The functional maps framework \cite{ovsjanikov2012functional} encodes shape matching as a compact function-to-function mapping $\mC$ between scalar-valued features $\mFx$ and $\mFy$ on two shapes $\cX$ and $\cY$. Using the LBO eigenfunctions $\mPhi_{\cdot}$ as a basis, the features are represented by coefficients $\mPhi_{\cdot}^{\dagger} \mF_{\cdot}$. The functional map $\mC$ is recovered by minimizing an energy that enforces feature preservation and structural regularity:

\vspace{-5pt}
\begin{align}
    \Edata(\mC) & = \| \mC \mPhix^{\dagger} \mFx - \mPhiy^{\dagger} \mFy \|_F^2, \\
    \mCxy & = \argmin_{\mC} \Edata(\mC) + \lambda \Ereg(\mC) 
    \label{eq:fmap}
\end{align}

\noindent where $\Edata$ promotes feature preservation, and $\Ereg$ encourages structural properties, implicitly promoting isometry \cite{ren2019structure, donati2022DeepCFMaps}. Subscripts $_{\cX/\cY}$ are omitted for general shapes.

The \emph{complex} functional maps $\mQ$ extend the FMaps concept to estimate correspondences between tangent vector fields (versus scalar functions) \cite{donati2022CFMaps}. The optimization structure remains similar to Eq.~\ref{eq:fmap}, but it uses vector-valued features (versus scalar-valued) and the complex eigenbasis of the Connection Laplacian (versus the LBO). 
The key benefit of $\mQ$ is its integration of surface orientation, making it suitable for disambiguating intrinsic shape symmetries. Both $\mC$ and $\mQ$ are small matrices, providing a highly compact representation compared to pointwise permutation matrices.

\paragraph{Equivariance and Vector Neurons.}
Equivariance is a fundamental property where a function $\cF$ maintains the structure of a transformation $\mR$ applied to its input $\vu$:
$\cF(\vu \mR) =  \cF(\vu) \mR,~~~\forall \mR \in \cT$,
where $\cT$ is a transformation group (e.g. 3D rotations). Our work focuses on $\rmSO{3}$-equivariance due to its relevance in 3D vision. The vector neuron (VN) framework \cite{deng2021vn} provides a general method for constructing SO(3)-equivariant networks. The core idea is to extend traditional scalar neurons to 3D vectors, enabling a simple, explicit mapping of SO(3) actions to the hidden feature space. This allows for straight-forward generalization of classic operations, such as the VN-Linear layer:
\begin{align*}
\VNLIN (\vu \mR) & = \mM (\vu \mR)  \\ &= (\mM \vu) \mR = \VNLIN (\vu) \mR, 
\end{align*}
where $\vu \in \bbR^{c\times 3}$ is the input with $c$ feature channels and $\mM \in \bbR^{c\times c}$ are the learnable parameters. 
The additional dimension of size $3$ is termed the \emph{VN dimension}. Since invariance is a specific case of equivariance where $\cF(\vu \mR) = \cF(\vu)$, our network learns $\mathrm{SO}(3)$-equivariant latent features that can be easily converted into invariant features to predict rotation-invariant correspondences.

\begin{figure}[t!]
    \centering
    \vspace{-0.7cm}
    \hspace*{0.3cm}
    \includegraphics[width=1.0\columnwidth,trim=120 170 30 7, clip]{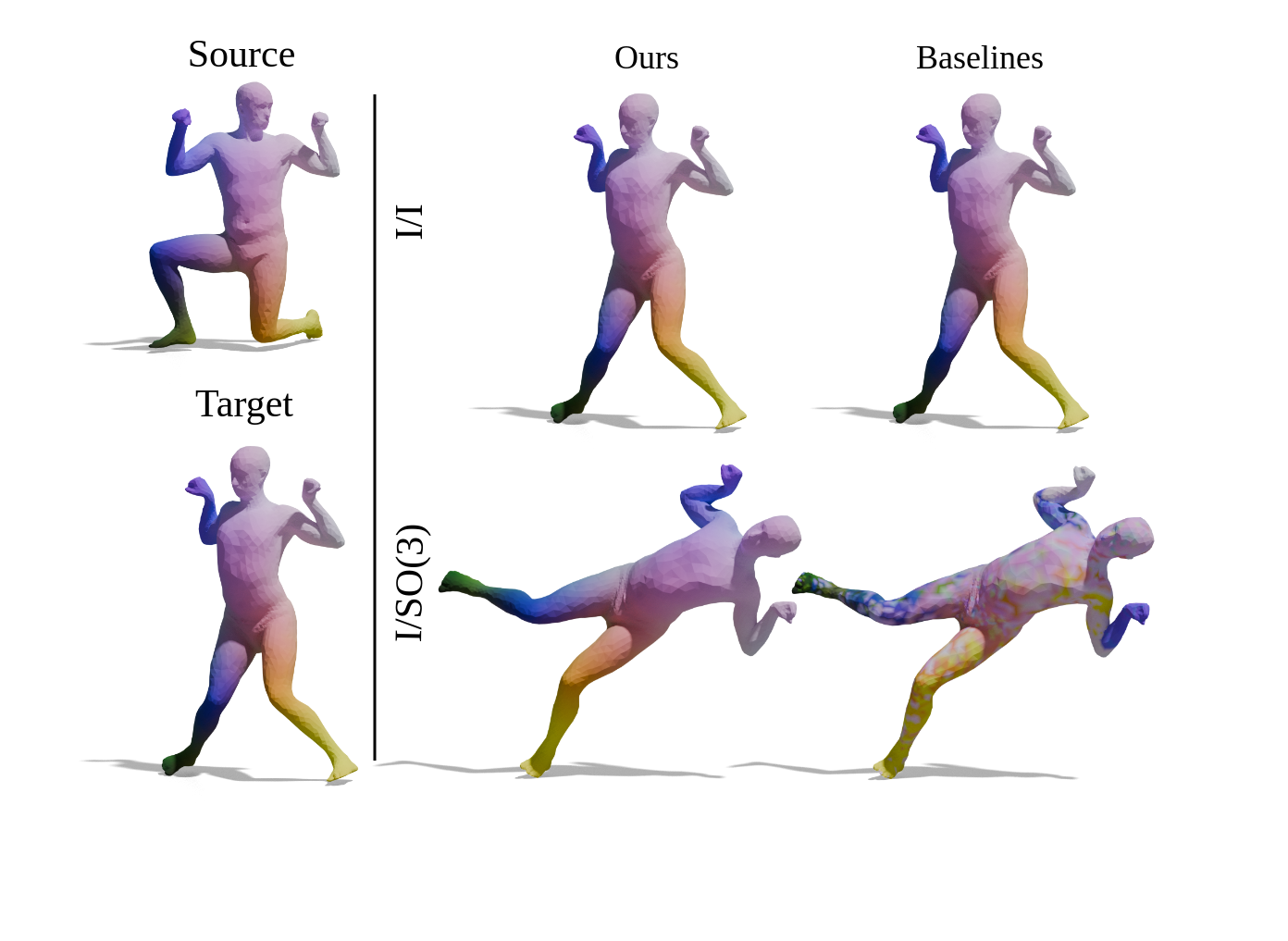}
    \caption{\textbf{SO(3)-invariant correspondences.} Our method is unaffected by shape orientations. In contrast, baselines (represented by \duofm) perform well only when train and test shapes are pre-aligned (I/I), failing dramatically on unseen rotations (I/SO(3)).}
    \label{fig:SO3-inv}
\end{figure}

\begin{figure*}[th!] 
    \centering
    \vspace{-0.9cm}
    \hspace*{-0.8cm}
    \includegraphics[width=2.2\columnwidth,trim=0 60 0 30, clip]{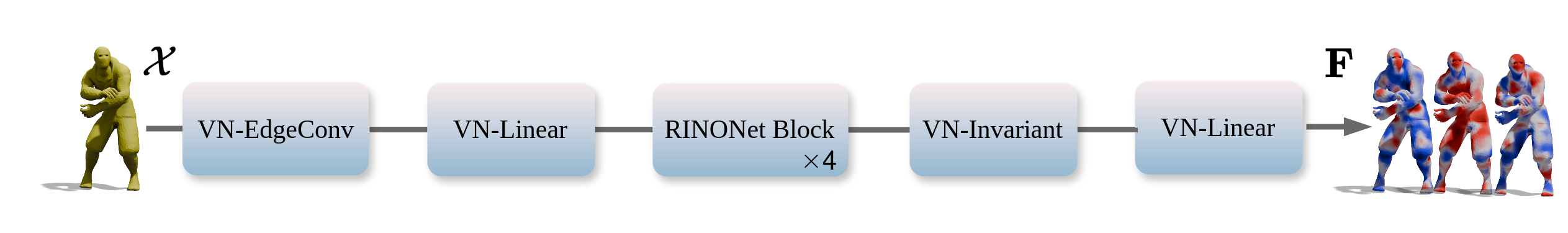}
    \caption{We propose \textbf{\netname}, which learns SO(3)-invariant features $\mF$ from input shape $\cX$. Our novel network inherits all nice properties from DiffusionNet~\cite{sharp2022diffusionnet}, and learns smooth, high-quality per-point features, while additionally retaining invariant to SO(3) actions applied to input shapes by employing \emph{vector neurons} as the feature representation in our hidden layers. The \netname~has a simple structure and contains four consecutive \textbf{\netname~Blocks} at its core, which is combined with VN-EdgeConv and VN-linear layers to achieve the desired I/O dimensionality. The VN-invariant layer is employed to convert SO(3)-equivariant features to invariant ones.}
    \label{fig:network}
    \vspace{-0.3cm}
\end{figure*}

%% file: sec/3_network.tex
\section{SO(3)-Invariant Manifold Learning}
\label{sec:vn-diffnet}

\paragraph{Motivation.}
We seek an ideal 3D shape descriptor characterized by three properties: invariance to extrinsic shape embedding; effective exploitation of surface geometric structure, and robustness to noise and topological variations.
The current standard, \diffusionnet~\cite{sharp2022diffusionnet}, fails to satisfy all these criteria simultaneously. Its reliance on the extrinsic \xyz~coordinates requires either pre-alignment or the use of manually designed (intrinsic) descriptors as input.
To overcome these limitations and realize a purely data-driven, geometric feature extractor, we propose \netname. Our approach maintains feature diffusion on the surface as in \cite{sharp2022diffusionnet}, but integrates VN representations to enforce SO(3)-equivariance (cf. Fig.~\ref{fig:feat}). This integration required a careful architectural redesign to ensure rotational equivariance across all spatial diffusion and gradient operations.

\vspace{-0.3cm}
\paragraph{Overview.}
\netname~starts with a VN-EdgeConv layer \cite{wang2019dgcnn, deng2021vn}, converting the input xyz coordinates $\mV \in \bbR^{n\times 3}$ into features $\bbR^{n\times c\times 3}$ which serves as an aggregation of neighborhood information.
The core structure is composed of four consecutive \netname~blocks, after which we append a VN-invariant layer (\cite[Sec.3.5]{deng2021vn}) to obtain $\rmSO{3}$-invariant features and, finally, a VN-linear layer to obtain output features with desired dimensionality.
Notice that internally our \netname~remains SO(3)-equivariant until the VN-invariant layer (cf. Fig.~\ref{fig:network}). %
In the next, we present the design insights of our novel \emph{\netname~block} and state its property of in-baked SO(3)-equivariance in Thm.~\ref{lemma:network_equivariance}. 
The full architecture is illustrated in Fig.~\ref{fig:network}.
For details please refer to the supp.mat..

\subsection{\netname~Block}
\label{subsec:vn-diffnetblock}
Our \netname~block consists of three main modules: a diffusion layer, a gradient layer, and an MLP (cf. Fig.~\ref{fig:vn-diffblock}). All modules are designed to be SO(3)-equivariant by harnessing the VN representation. The hidden states of our network, $\vu$ and $\vd$, become vector-valued with dimensionality $\bbR^{n\times c\times 3}$ (instead of scalar-valued $\bbR^{n \times c}$), with $n$ vertices and $c$ feature channels. Next, we explain these modules.

\vspace{-0.3cm}
\paragraph{VN-Diffusion Layer.}

The goal of this layer is to perform feature diffusion using a learned diffusion time $t$ for each feature channel, which governs the size of the spatial support of the resulting features. 
Naively, feature diffusion is analogous to applying a linear operator $\cH_t$ on an input feature $\vu$, yielding $\vh \coloneqq \cH_t(\vu)$. 
However, to ensure SO(3)-equivariance, we must apply the same diffusion time $t$ across all three VN dimensions of the input feature $\vu$. 
This crucial design, together with the linearity of the diffusion operator, guarantees the SO(3)-equivariance, namely $\cH_t(\vu \mR) = \cH_t(\vu) \mR$. 
Furthermore, to enhance the model's ability to adjust the size of its spatial support, the diffusion time $t$ is learned per feature channel. 
Note that although $t$ is a learnable parameter updated during back-propagation, it remains fixed during any single feed-forward pass.

\vspace{-0.3cm}
\paragraph{VN-Gradient Layer.}
The idea of this layer is to compute spatial gradient features $\ve$ from a diffused input feature $\vh$ to learn radially asymmetric filters, a crucial capability for expressive geometric representation. Similar to \cite{sharp2022diffusionnet}, we begin by computing the intrinsic spatial gradient $\vw \coloneqq \mathcal{G}(\vh)$, which is then expressed in the local frame of each vertex as complex numbers. The gradient is transformed by a learnable (complex) matrix $\mA \in \mathbb{C}^{c \times c}$, which is further aggregated by $\re(\overline{\vw} \odot \mA\vw)$. However, this naive approach breaks the SO(3)-equivariance, as the aggregation operation does not commute with rotation, meaning $(\overline{\vw}\mR \odot \mA\vw\mR) \neq (\overline{\vw} \odot \mA\vw)\mR$.
Our idea is an additional summation along the VN dimension to obtain the aggregated feature $\vf$:

\begin{equation}
    \label{eq:feat_grad_raw}
    \vf \coloneqq \mathrm{sum} (\re(\overline{\vw} \odot \mA\vw), \mathrm{dim}=1).
\end{equation}

where $\re$ extracts the real part of a complex number.
In the supp.mat., we demonstrate another equivalent formulation of Eq.~\ref{eq:feat_grad_raw} and prove the following result:

\begin{theorem}
\label{lemma:feat_grad_raw_invariance}
    The feature $\vf \in \bbR^{n \times c}$ computed as in Eq.~\ref{eq:feat_grad_raw} is $\rmSO{3}$-invariant.
\end{theorem}

 After stabilizing training by passing $\vf$ through a $\tanh$ layer $\vg \coloneqq \tanh (\vf)$, we obtain the final equivariant features $\ve$ by an element-wise multiplication (with broadcasting) of the invariant feature $\vg \in \mathbb{R}^{n \times c}$ with the normalized equivariant feature $\vh \in \mathbb{R}^{n \times c \times 3}$: $\ve \coloneqq \vg \odot \vh$. This formulation maintains the network equivariance as much as possible.

\begin{figure}[t!] 
\vspace*{-0.3cm}
\hspace*{-1.6cm}  
    \centering
    \includegraphics[width=1.35\columnwidth,trim=0 0 0 0, clip]{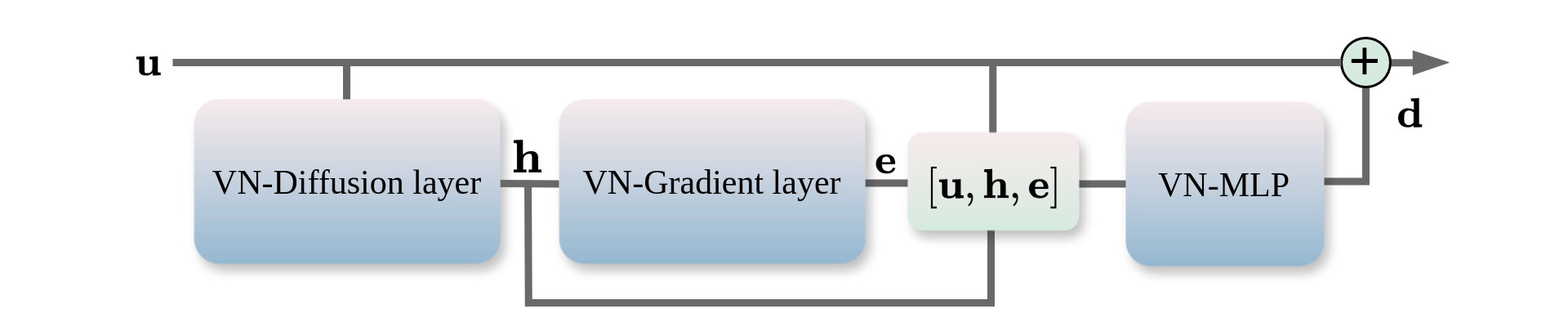}
    \caption{\textbf{Our \netname~block} is the core of our \netname, and it consists of three main modules: a VN-Diffusion layer, a VN-Gradient layer, and a VN-MLP. In contrast to the original \diffusionnet~block~\cite{sharp2022diffusionnet}, ours is SO(3)-equivariant by design.}
    \label{fig:vn-diffblock}
\vspace*{-0.3cm}
\end{figure}

\vspace{-0.3cm}
\paragraph{VN-MLP.}
An MLP consists of linear layers followed by $\relu$ non-linearities, we disable the bias terms in the linear layers, since $\mM(\vu\mR)+\vb \neq (\mM\vu + \vb)\mR$ and swap the $\relu$ with its VN counterpart, to enable equivariance.

\vspace{-0.3cm}
\paragraph{Remark.}
A \netname~block can be summarized as:

\begin{equation}
\label{eq:vn-diffnet_block}
    \vd = \VNMLP([\vu, \vh, \ve]) + \vu,
\end{equation}
where the input feature $\vu$, the equivariant diffusion feature $\vh$ and the equivariant gradient feature $\ve$ are concatenated together along the $c$ dimension before going through a $\VNMLP$.

\begin{theorem}
\label{lemma:network_equivariance}
    The \netname~block is $\rmSO{3}$-equivariant and the whole \netname~is $\rmSO{3}$-invariant.
\end{theorem}

\noindent Please refer to the supp.mat. for a proof.

As discussed in~\cite{sharp2022diffusionnet}, a complex matrix $\mA$ is beneficial for disambiguating intrinsic symmetry due to its awareness of the shape orientation. 
In our experiments, we found that symmetrically flipped correspondences still appear when the xyz coordinates of unaligned shapes are fed into the network (cf. Fig.~\ref{fig:sym-flip} \& Tab.~\ref{tab:sym}). Hence, we propose to exploit CFMaps, which can only encode orientation-preserving maps in theory, to fully eliminate intrinsic symmetry and regularize learned features.

\begin{figure}[t!]
    \vspace{-0.8cm}
    \hspace*{-0.1cm}
    \includegraphics[scale=0.15, trim=400 230 40 8, clip]{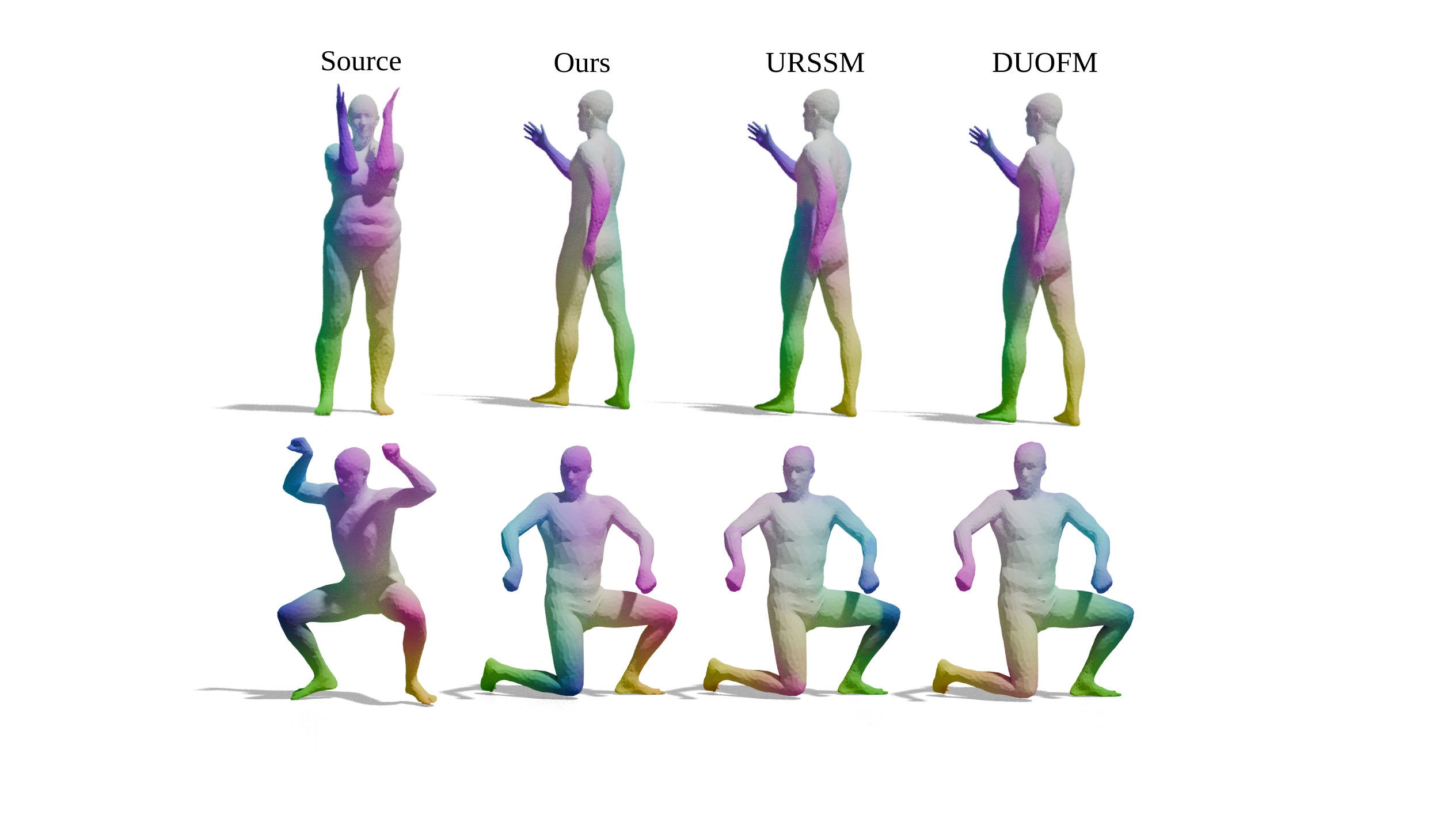}
    \caption{\textbf{Intrinsic symmetry.} Ours fully resolves symmetry, while baselines suffer. \textbf{Top:} partial left-right (legs) and front-back (belly-back) flips. \textbf{Bottom:} full left-right flips (arms and legs).}
    \label{fig:sym-flip}
\end{figure}

%% file: sec/4_method.tex
\section{Unsupervised SO(3)-Invariant Matching}
\label{sec:matching-method}

\begin{figure*}[h!t!] 
\vspace*{-0.8cm}  
\hspace*{-0.5cm}
    \centering
    \includegraphics[width=2.0\columnwidth,trim=0 60 0 0, clip]{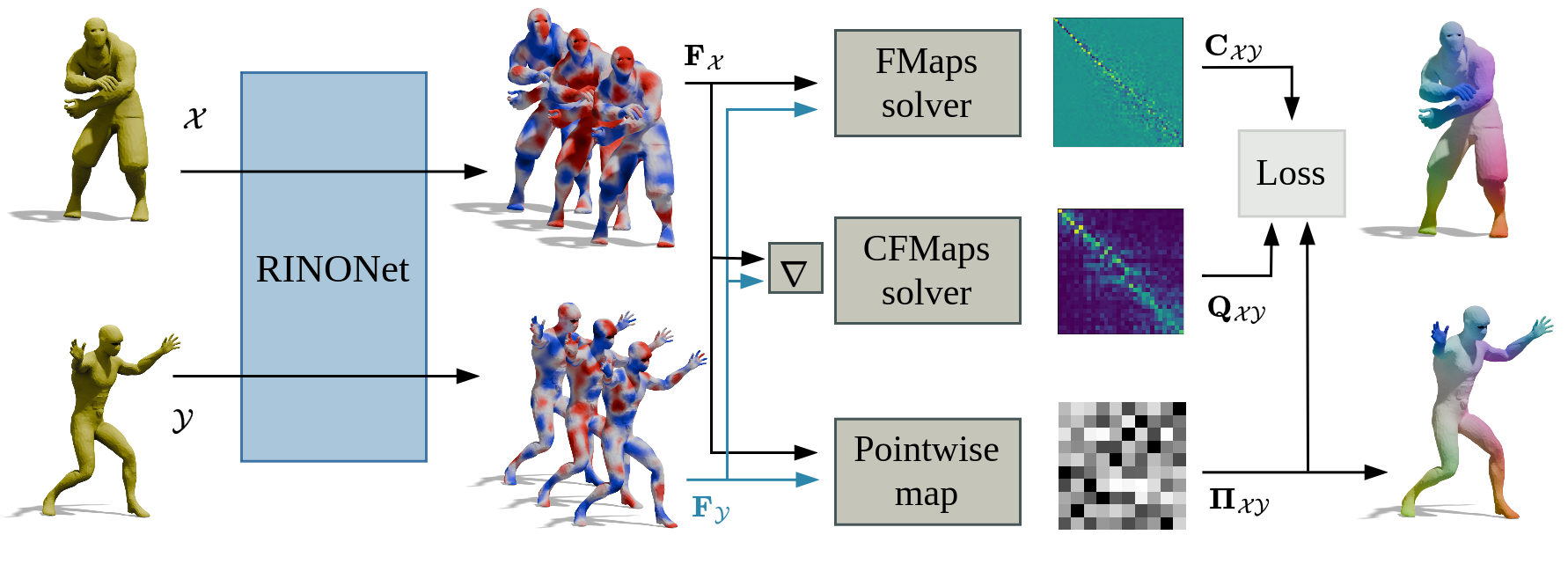}
    \caption{\textbf{Overview of \methodname.} Our novel \netname~is used to extract SO(3)-invariant features $\mFx$ and $\mFy$ from input shapes $\cX$ and $\cY$. The features are further processed in three non-learnable but differentiable blocks: the FMaps solver to compute the functional map $\mCxy$, the CFMaps solver to compute the complex functional map $\mQxy$ (based on feature gradients), and the pointwise map block to compute the soft pointwise map $\mPixy$ based on feature similarity. Finally, our unsupervised loss prompts structural properties of $\mCxy, \mQxy$ and consistency between different map representations.}
    \label{fig:pipeline}
\vspace*{-1em}  
\end{figure*}

\paragraph{Overview.} 
We integrate our \netname~into a novel triple branch matching framework and propose \methodname, the first unsupervised SO(3)-invariant approach for non-rigid shape matching. The process starts by passing two \mbox{input shapes}, $\cX$ and $\cY$ (with $\nx$ and $\ny$ vertices, respectively), through the network to obtain their (flattened) SO(3)-invariant features $\mFx \in \bbR^{\nx \times 3c}$ and $\mFy \in \bbR^{\ny \times 3c}$ in a Siamese fashion (cf. Fig.~\ref{fig:feat}). 
These expressive features are then leveraged to compute three distinct map representations for matching: a (soft) pointwise map $\mPixy$, an FMap $\mCxy$, and a CFMap $\mQxy$. All of them provide important structural properties for the matching that we aim to exploit with a careful coupling of the losses. During \emph{inference}, \methodname~can match shapes in a single forward pass by Euclidean nearest neighbor search in the learned feature space.
Our full pipeline is shown in Fig.~\ref{fig:pipeline}.

\begin{table}[t!]
\vspace{-0.3cm}
 \begin{adjustwidth}{-0.0cm}{}
\centering
    \small
\begin{tabular}{lrrrr}
\toprule
Train/Test & \textbf{I/I} & \textbf{I/SO(3)} & \textbf{SO(3)/SO(3)} & \textbf{Y/Y}\\
\midrule
\consistfm & 5.4 & 58.7 & 9.1 & \cellcolor{yellow!30} 5.4 \\
\duofm & 11.7 & 41.4 & 25.2 & 16.5 \\
\ulrssm & 4.8 & 62.1 & 24.8  & \cellcolor{yellow!30} 7.9 \\
\sms & \cellcolor{yellow!30} \textbf{4.6} & 58.5 & 57.6  & 39.3 \\
\sfm & 5.6 & 62.6 & 9.9  & \cellcolor{yellow!30} 7.3 \\
\hybridfm & 5.6 & 61.8 & 26.4  & \cellcolor{yellow!30} 5.3 \\
{Ours} & \textbf{4.6} & \textbf{4.6} & \textbf{4.6}  & \textbf{4.6}\\ 
\bottomrule
\end{tabular}
\caption{ \textbf{SO(3)-invariance on SMAL.} The rotation type is indicated by: {I} (aligned), {SO(3)} (fully random rotation), and {Y} (y-axis rotation). \emph{e.g.} {I/SO(3)} means aligned shapes for training and fully randomly rotated shapes for testing.  We report mGeoErr ($\downarrow$) and ours consistently outperforms baselines, especially on unseen rotations. Best results are \textbf{bold}, and \colorbox{yellow!30}{highlighted} cells indicate the experiment settings used in baseline papers.}
\label{tab:exp-so3}
\end{adjustwidth}

\vspace{-0.3cm}
\end{table}

\paragraph{Map Representations.}
Our method derives three complementary map representations from the $\mathrm{SO}(3)$-invariant features $\mFx$ and $\mFy$ generated by \netname. The soft pointwise map is computed via softmax on the feature similarity: $\mPixy = \mathrm{Softmax}(\mFx \mFy^T / \tau)$, where $\tau$ is a temperature parameter controlling the match entropy. Simultaneously, $\mFx$ and $\mFy$ are fed into two non-learnable, differentiable blocks: the \emph{FMaps} and the \emph{CFMaps Solver}. These blocks efficiently solve convex problems to estimate $\mCxy$ and $\mQxy$ (cf. Eq.~\ref{eq:fmap}). 
Our insight is twofold: first, $\mQxy$ inherently encodes only orientation-preserving maps, which is essential for disambiguating intrinsic shape symmetry \cite{donati2022CFMaps}. Second, while $\mCxy$ enforces feature consistency, $\mQxy$ additionally enforces consistency of the features' first-order derivatives (obtained via the gradient operator $\nabla$), leading to more accurate features and correspondences.%

\paragraph{Unsupervised Loss.}
\label{subsec:loss}

Our unsupervised objective $\Ltotal$ is composed of three terms: a novel coupling loss $\Lcouple$, a structural loss $\Lstruct$, and a contrastive loss $\Lcontrast$.
The terms $\Lstruct$ and $\Lcontrast$ are inspired by \cite{roufosse2019unsupervised, cao2024revisiting, donati2022CFMaps, donati2022DeepCFMaps}, which encourage key properties of $\mC$ and $\mQ$, and discriminative per-vertex features (see supp.mat. for details). 

Our contribution is the $\Lcouple$, which links the three map representations: $\mPi, \mC, \mQ$, derived from our invariant features. We enforce consistency between the soft pointwise map $\mPi$ and its functional pullbacks $\mC$ and $\mQ$ respectively. The coupling loss is defined as:
\begin{equation}
\label{eq:loss_couple_main}
    \Lcouple = \Lpc + \Lpq
\end{equation}
where $\Lpc$ is the consistency between the pointwise map and $\mC$, and $\Lpq$ is the proposed term enforcing consistency between the pointwise map and the orientation-aware complex functional map $\mQ$. This targeted coupling is crucial for simultaneously achieving highly accurate features and robust symmetry disambiguation without requiring excessive coupling that can jeopardize training. Note that although the coupling of $\mPi$ and $\mC$ has been explored in \cite{sun2023spatially, attaiki2023clover, cao2023unsupervised}, we are the first to study the coupling of all three map representations.
Following our novel pipeline and loss, our estimated $\mC$ is an orientation-preserving isometry akin to the results in~\cite{donati2022DeepCFMaps} (cf. supp.mat.).

%% file: sec/5_experiment.tex
\begin{table}
\vspace{-0.3cm}
\hspace*{-0.2cm}
    \setlength{\tabcolsep}{4pt}
    \small
    \centering
    \begin{tabular}{@{}lcccccc@{}}
    \toprule
      & \multicolumn{3}{c}{\textbf{\wks}} & \multicolumn{3}{c}{\textbf{\xyz}} \\ 
        \cmidrule(lr){2-4} \cmidrule(lr){5-7}
    \multicolumn{1}{l}{}  & \multicolumn{1}{c}{\textbf{SMAL}}   & \multicolumn{1}{c}{\textbf{DT4D}}  & \multicolumn{1}{c}{\textbf{FSCAN}}  & \multicolumn{1}{c}{\textbf{SMAL}}   & \multicolumn{1}{c}{\textbf{DT4D}}  & \multicolumn{1}{c}{\textbf{FSCAN}} \\ 
    \midrule
    \multicolumn{1}{l}{\consistfm}  & 23.4 &  \cellcolor{yellow!30} 9.3 & 3.6  & \multicolumn{1}{c}{{9.1}}& 7.4 & 50.5\\
    \multicolumn{1}{l}{\duofm}  & \colorbox{yellow!30}{\underline{6.7}}  & 15.8 & - & \multicolumn{1}{c}{25.2}  & 55.8  & - \\
    \multicolumn{1}{l}{\ulrssm}  & 29.6  & \cellcolor{yellow!30} 8.4 & 3.0 & 24.8 & 59.3 & 24.8  \\
    \multicolumn{1}{l}{\sms} & 30.6 & 58.6 & 12.1 & 57.6 & 64.9 & 54.0 \\
    \multicolumn{1}{l}{\sfm} & 29.9 & \cellcolor{yellow!30} 15.3 & 5.7 & 9.9 & {5.5} & 24.3 \\
    \multicolumn{1}{l}{\hybridfm} & 34.5 & \cellcolor{yellow!30} {\underline{5.4}} & 3.5 & 26.4 & 33.8 & {22.1} \\
    \multicolumn{1}{l}{Ours} & 28.7 & 9.8 & \underline{2.7} & \multicolumn{1}{c}{\textbf{4.6}} & \multicolumn{1}{c}{\textbf{5.3}}   & \textbf{2.5}\\
\bottomrule
    \end{tabular}
    
    \caption{\textbf{Non-isometric and raw scan matching.} We report mGeoErr ($\downarrow$) and highlight the \textbf{best} and \underline{2nd best} results. Ours outperforms all baselines regardless of the input types. \colorbox{yellow!30}{Highlighted} cells indicate the experiment settings used in baseline papers.} 
    \label{tab:non-isometry}
\vspace*{-1em}  
\end{table}

\section{Experiments}
\label{sec:exp}

\begin{figure*}[t!]
    \vspace{-0.2cm}
    \hspace{-0.5cm}
    \includegraphics[width=2.3\columnwidth,trim=320 400 40 75, clip]{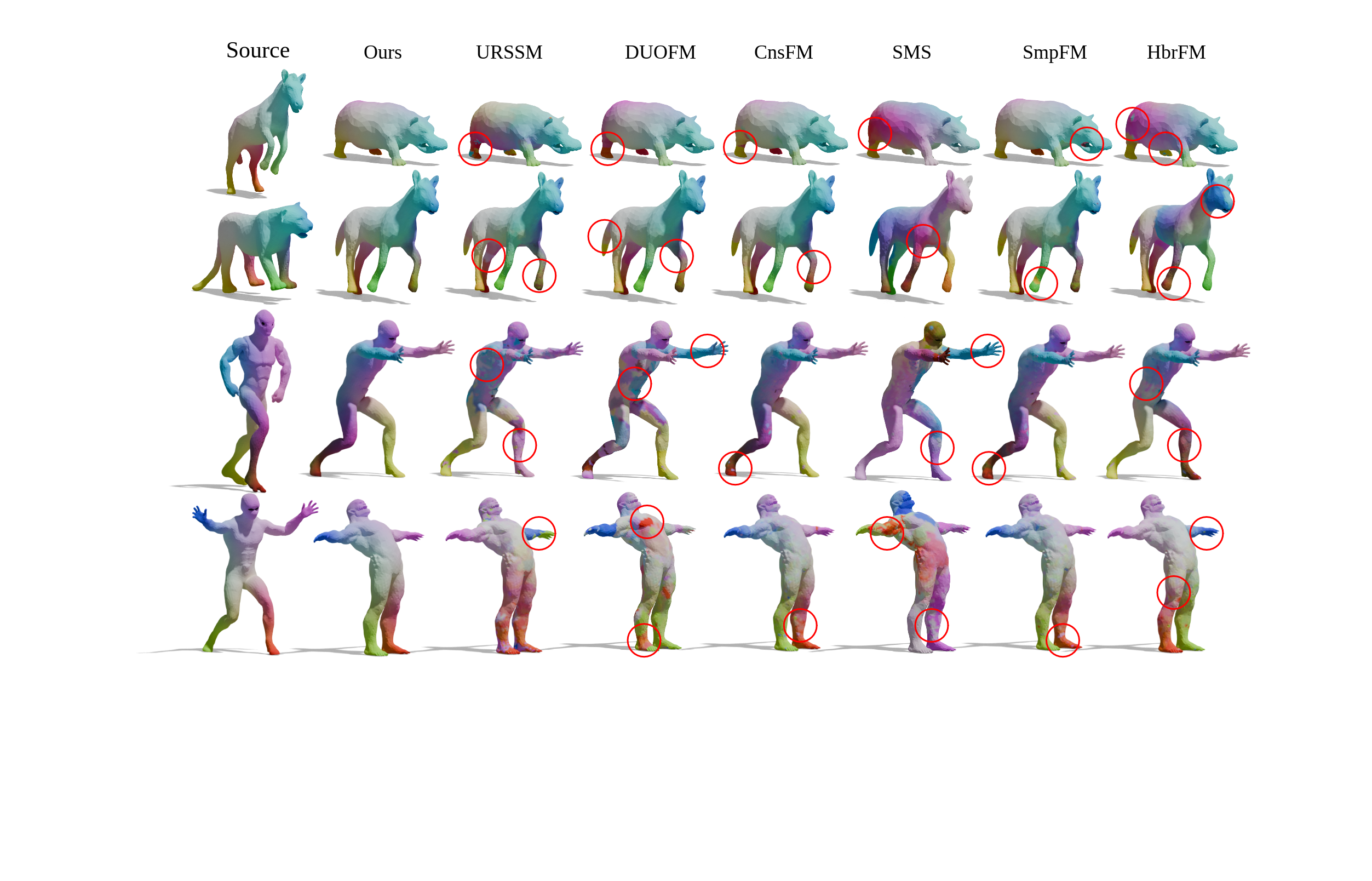}
    \caption{\textbf{Non-isometric matching on SMAL and DT4D.} Our method outperforms all baselines; it is invariant to extrinsic rigid poses (no pre-alignment) and robust to intrinsic shape symmetry by directly processing raw shape geometry.}
    \label{fig:non-iso}
    \vspace*{-1em}
\end{figure*}

\paragraph{Baselines and Metrics} 
The proposed method is extensively compared against state-of-the-art unsupervised shape matching techniques, including the deep complex functional maps method \duofm \cite{donati2022DeepCFMaps}, the deep map coupling method \ulrssm \cite{cao2023unsupervised}, the cycle-consistent multi-matching method \consistfm \cite{sun2023spatially}, the deep non-rigid ICP method \sms \cite{cao2024spectral}, the hybrid functional maps method \hybridfm \cite{bastian2023hybrid}, and the memory-efficient functional maps method \sfm \cite{magnet2024memory}.
The focus on competitive unsupervised baselines is due to their superior performance over axiomatic approaches and their practical relevance given the scarcity of 3D assets with ground truth annotations. For a fair evaluation, all shapes are subject to random rotation during both training and testing under \xyz~as input unless otherwise mentioned, and no post-processing \cite{melzi2019zoomout,cao2023unsupervised,vigano2025nam,vestner2017efficient}
is applied, as all compared methods can be post-processed.

For all experiments, we adhere to the Princeton protocol \cite{shilane2004princeton}, where the standard evaluation metric reported is the mean geodesic error (\textbf{mGeoErr}) with respect to the ground truth correspondence.

\begin{table}[]
\vspace{-0.3cm}
\hspace*{-0.2cm}
    \centering
    \begin{tabularx}{0.5\textwidth}{lXXXXXX}
    \toprule
      & \multicolumn{3}{c}{\textbf{FAUST}} & \multicolumn{3}{c}{\textbf{SCAPE}} \\ 
        \cmidrule(lr){2-4} \cmidrule(lr){5-7}
      & \textbf{E} & \textbf{ES}  & \textbf{\#Flips} & \textbf{E} & \textbf{ES}  & \textbf{\#Flips} \\ 
    \midrule
    \consistfm      & 4.4 &  2.8 &  25        &  {27.2} &  {7.4} &  {249}         \\
    \duofm          & 28.0 &  7.2 &  204      &  {28.3} &  {8.9} &  {232}        \\
    \ulrssm         &  30.0 &  7.8 &  240      &  {26.9} &  {7.7} &  {245}        \\
    \sfm            &  3.0  &  2.7 &  1        &  {29.9} &  {7.9} &  {63}            \\
    \hybridfm       &  29.9 &  7.2 &  227      &  {27.2} &  {8.1} &  {236}        \\
    Ours            &  \textbf{1.6} &  \textbf{1.5} &  \textbf{0}      &  \textbf{2.0} &  \textbf{2.0}  &  \textbf{0}            \\
    \bottomrule
    \end{tabularx}

    \caption{\textbf{Symmetry analysis}. We compute mGeoErr ($\downarrow$)  by allowing only symmetry-free correspondences (col. \mbox{\textbf{E}}) and additionally the symmetric ones (col. \textbf{ES}), and report the number of symmetrically flipped estimations out of 400 test pairs (col. \mbox{\textbf{\#Flips}}).}
    \label{tab:sym}
\vspace*{-1em}  
\end{table}

\subsection{SO(3)-Invariant Matching}

To assess SO(3)-invariant features, we utilized the remeshed SMAL dataset \cite{Zuffi:CVPR:2017, ren2018continuous} across four train/test configurations based on shape rotation: $\mathrm{I/I}$ (aligned), $\mathrm{Y/Y}$ (y-axis rotated), $\mathrm{SO}(3)/\mathrm{SO}(3)$ (fully randomly rotated), and $\mathrm{I/SO}(3)$ (aligned training, fully rotated testing). Due to the non-isometry presented in SMAL, all baselines (except \duofm) use \xyz~coordinates as input. 
As shown in Tab. \ref{tab:exp-so3}, while all baselines perform well in the $\mathrm{I/I}$ setting, their performance significantly degrades in the $\mathrm{Y/Y}$ and $\mathrm{SO}(3)/\mathrm{SO}(3)$ cases, and crucially, all baselines fail severely in the challenging $\mathrm{I/SO}(3)$ configuration (unseen rotation at test time). 
This confirms that existing methods lack robustness to arbitrary rotational changes and rely on shape alignment. While data augmentation can partially address this at the cost of training effort, our method achieves SO(3)-invariant features by design, thus completely eliminating the need for cumbersome and inefficient rotational augmentation. See Fig.~\ref{fig:SO3-inv} for visual examples.

\subsection{Non-Isometric Matching}
To evaluate our method's robustness to non-isometric deformations, we utilize the SMAL and DT4D \cite{Zuffi:CVPR:2017, magnet2022smooth, ren2018continuous} datasets, with quantitative results reported in Tab.~\ref{tab:non-isometry}. 
\methodname~outperforms all baselines, regardless of their input type. 
This superior performance directly addresses the known limitation of spectral descriptors like \wks, whose near-isometric dependency makes them unstable under non-isometric deformation. 
By leveraging SO(3)-invariance, \methodname~is able to directly process raw 3D geometry. 
Qualitative results in Fig.~\ref{fig:non-iso} further show that our approach successfully yields high-quality, flip-free correspondences, whereas baselines suffer from non-smoothness and intrinsic left-right symmetry. 
We also notice the effectiveness of \sfm~when using \xyz, which is undiscovered in its original paper. This is likely due to its effective architecture, which limits the use of SO(3)-variant features only in the differentiable Zoomout module, thereby avoiding detrimental coupling with low-quality pointwise maps. 
Methods like \duofm~and \ulrssm~struggle to learn meaningful correspondences under challenging random rotations. 
Additionally, results on the FAUSTSCAN \cite{bogo2014faust} (denoted as FSCAN) are reported, with further discussion in the supp.mat..

\subsection{Disambiguating Intrinsic Symmetry}
\label{sec:sym}

This experiment analyzes robustness against intrinsic symmetry using SCAPE and FAUST \cite{anguelov2005scape, bogo2014faust}, with results in Tab.~\ref{tab:sym}. Allowing for symmetrically flipped matches significantly improves baseline accuracy, indicating that a substantial portion of their correspondence error is due to symmetry confusion (symmetric flips in $\sim$200 of 400 test pairs). Our method inherently avoids this issue due to the integration of CFMaps. While \duofm~also uses CFMaps, it still confuses symmetry because its features are not SO(3)-invariant, making them dependent on extrinsic embeddings. Although \diffusionnet's spatial gradient features can offer some aid via anisotropic diffusion \cite[Fig.~5]{sharp2022diffusionnet}, we found it insufficient to fully resolve intrinsic symmetry, especially under random shape rotations.

\subsection{Partial Shape Matching}

The severe non-isometry inherent in partial shapes often forces baselines to switch their input features to \xyz~coordinates. Following the training strategy of \cite{cao2023unsupervised}, we first pretrain \methodname~on a collection of four full shape datasets (DT4D, SMAL, FAUST, SCAPE) for two epochs. The model is then fine-tuned on the SHREC16-Partiality (CUTS and HOLES)~\cite{cosmo2016shrec,ehm2023geometrically} for 500 epochs, without any post-processing for evaluation. 
As shown in Tab.~\ref{tab:partial}, \methodname~outperforms strong partial matching baselines, even the supervised method \echomatch. This superior performance is a strong indication that the ability to handle unaligned partial shapes, a feature built into our design, is crucial for successful partial shape matching (cf. Fig.~\ref{fig:partial}).

\begin{figure}
    \centering
    \vspace{-0.3cm}
    \includegraphics[trim={240, 290, 10, 240}, clip, width=0.7\textwidth]{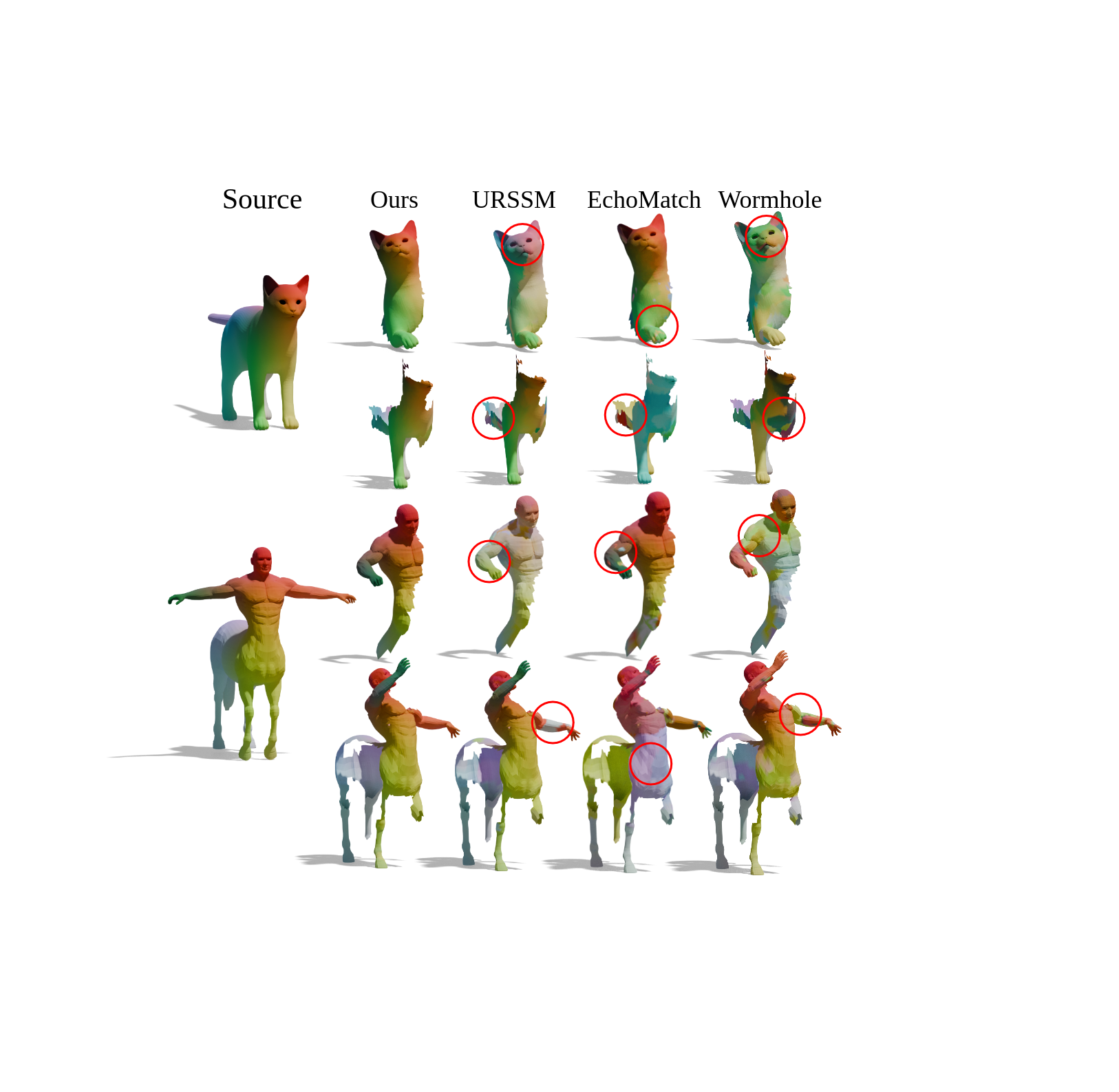}
    \caption{\textbf{Qualitative results on SHREC16-Partiality.}}
    \label{fig:partial}
\end{figure}

\begin{table}[t]
\centering
\begin{tabular}{lcccc}
\toprule
 & \ulrssm & \echomatch & \wormhole & Ours \\
\midrule
\textbf{CUTS} &  42.1 & 12.7 & 48.2 &  \textbf{6.7}\\
\textbf{HOLES} & 29.8 & 57.1 & 26.9 &  \textbf{12.09} \\
\bottomrule
\end{tabular}
\caption{\textbf{mGeoErr ($\downarrow$) on SHREC16-Partiality.} Ours outperforms all baselines, including the supervised method \echomatch.}
\label{tab:partial}
\vspace{-0.3cm}
\end{table}

\subsection{Robustness under Noise}
To analyze robustness to noise, we evaluate ours and baselines on ten SCAPE shapes corrupted by Gaussian noise with variance $\sigma \in [1\text{e-}3, 1\text{e-}2]$. 
As shown in Fig.~\ref{fig:noise-curve}, while all baselines perform acceptably up to $\sigma=4\text{e-}3$, ours demonstrates superior robustness, with performance noticeably degrading only after $\sigma=6\text{e-}3$, beyond which the degradation rate is approximately linear. 
At this high noise level, ours maintains accurate correspondences where baselines yield spurious results  (cf. Fig.~\ref{fig:noise-vis}). We attribute this advantage to two factors: 1) our design processes \xyz~coordinates directly, avoiding noise-sensitive shape descriptors like \wks; and 2) our \netname~has a significantly reduced number of trainable parameters, which acts as a powerful regularization effect, enhancing generalization to imperfect shapes.

\begin{figure}[t!]
    \includegraphics[trim={200 280 50 70}, clip, width=0.55\textwidth]{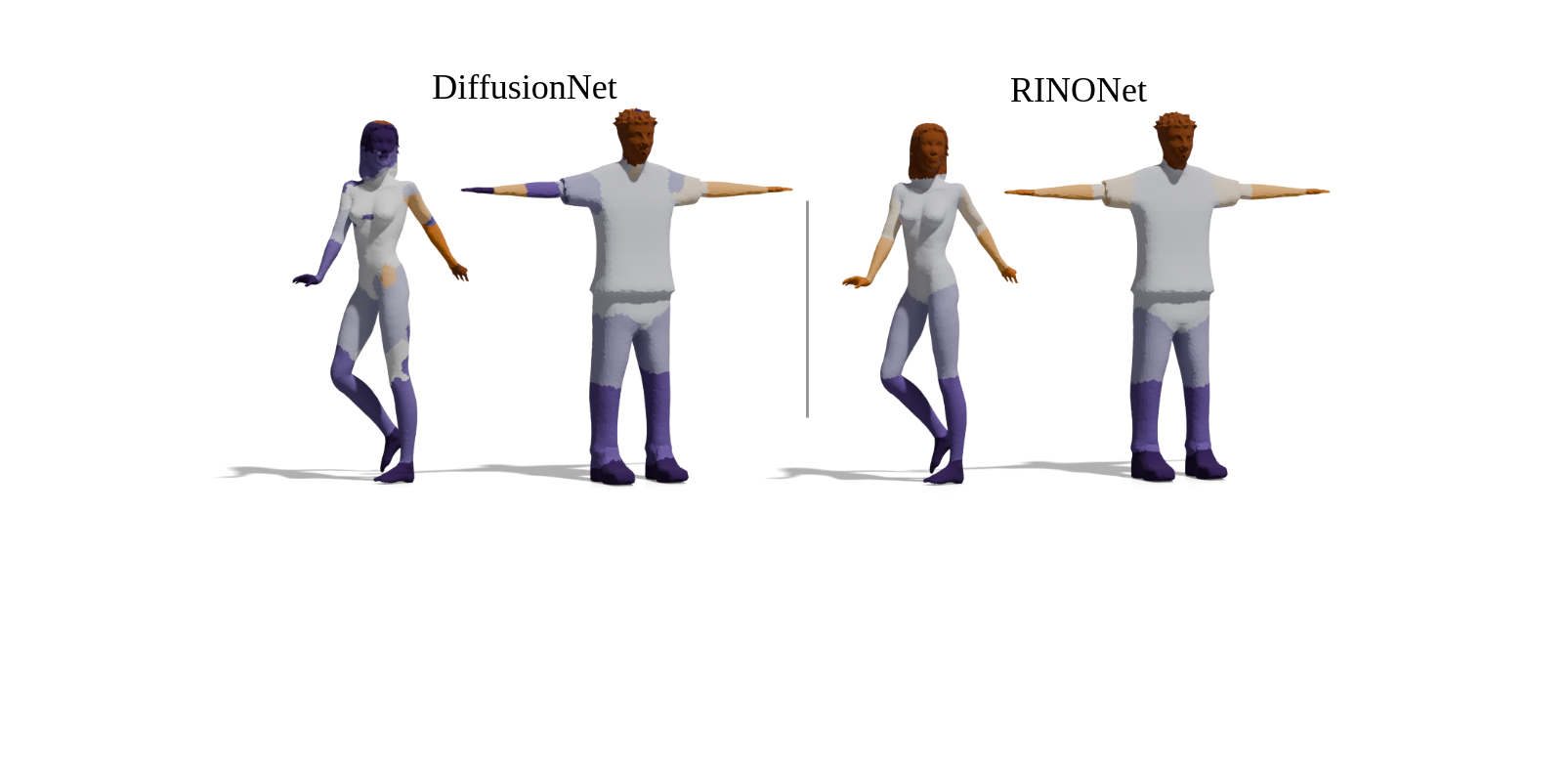}
    \caption{\textbf{Human Segmentation.} We show the segmentation results of \diffusionnet~and ours on the composite dataset\cite{maron2017cnn}. Our \netname~produces sharper segments.}
    \label{fig:seg}
    \vspace{-1em}
\end{figure}

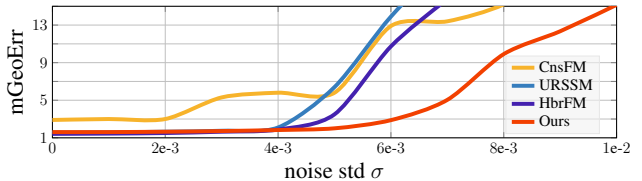
\begin{figure}[t!]
    \centering
    \hspace{-1.1cm}
    \input{figs/noise.tex}
    \vspace*{-14pt}
    \caption{\textbf{Robustness under noise perturbations.}}
    \label{fig:noise-curve}
\vspace*{-1em}  
\end{figure}

\subsection{Segmentation \& More}

As a proof-of-concept, we employ \netname~for human segmentation tasks using the composite dataset introduced in \cite{maron2017cnn}.
We compare against the SoTA \diffusionnet~\cite{sharp2022diffusionnet} and Fig.~\ref{fig:seg} shows the superior performance in the segmentation task, hence the promising potential of \netname~in a wider range of 3D understanding tasks.

\noindent Please refer to the supp.mat. for further studies on computational complexity, an ablation analysis of our loss and architecture, and experimental results on both isometric and non-manifold shapes.

%% file: figs/noise.tex
\newcommand{\pckLineWidth}{3pt}
\newcommand{\plotWidth}{\linewidth}
\newcommand{\plotHeight}{\linewidth}
\newcommand{\pckTitle}{\textsc{SMAL NOISE}}
\definecolor{cPLOT_purple}{RGB}{68, 33, 175}
\definecolor{cvprblue}{rgb}{0.21,0.49,0.74}
\definecolor{cPLOT_ours}{RGB}{242,64,0}
\definecolor{cPLOT_yello}{RGB}{247,179,43}

\pgfplotsset{%
    label style = {font=\LARGE},
    tick label style = {font=\large},
    title style =  {font=\LARGE},
    legend style={  fill= gray!10,
                    fill opacity=0.6, 
                    font=\large,
                    draw=gray!20, %
                    text opacity=1}
}
\begin{tikzpicture}[scale=0.5, transform shape]
	\begin{axis}[
		width=\plotWidth,
		height=\plotHeight,
        x=1.5cm,
        y=0.25cm,
		grid=major,
		legend style={
			at={(0.97,0.03)},
			anchor=south east,
			legend columns=1},
		legend cell align={left},
	ylabel={{\LARGE
        mGeoErr}},
        xmin=0,
        xmax=10,
        xlabel=
        \LARGE
        noise std $\sigma$,
        ylabel near ticks,
        xtick={0, 2, 4, 6, 8, 10},
        xticklabels={$0$, $2\text{e-}3$, $4\text{e-}3$, $6\text{e-}3$, $8\text{e-}3$, $1\text{e-}2$},
        ymin=1,
        ymax=15,
        ytick={1, 3, 5,  7, 9, 11, 13, 15},
        yticklabels={$1$, $~$, $5$, $~$, $9$, $~$, $13$, $~$},
	]
    
\addplot [color=cPLOT_yello, solid, smooth, line width=\pckLineWidth]
    table[row sep=crcr]{%
0.000 2.9 \\
1 3.0 \\
2 3.0 \\
3 5.3 \\
4 5.8 \\
5 5.8 \\
6 12.9 \\
7 13.4 \\
8 15.2 \\
9 18.9 \\
10 20.0 \\
    };
    \addlegendentry{\textcolor{black}{\consistfm}}
    
\addplot [color=cvprblue, solid, smooth, line width=\pckLineWidth]
    table[row sep=crcr]{%
0.000 1.56 \\
1 1.58 \\
2 1.65 \\
3 1.75 \\
4 2.1\\
5 6.4\\
6 13.9 \\
7 19.8 \\
8 22.6 \\
9 24.3 \\
10 23.4 \\
    };
    \addlegendentry{\textcolor{black}{\ulrssm}}
\addplot [color=cPLOT_purple, solid, smooth, line width=\pckLineWidth]
    table[row sep=crcr]{%
0.000 1.42 \\
1 1.45 \\
2 1.51 \\
3 1.65 \\
4 1.88\\
5 3.5\\
6 10.73 \\
7  15.75\\
8  20.04\\
9 22.38 \\
10 20.79 \\
    };
    \addlegendentry{\textcolor{black}{\hybridfm}}
\addplot [color=cPLOT_ours, smooth, line width=\pckLineWidth]
    table[row sep=crcr]{%
0.000 1.61 \\
1 1.61 \\
2 1.65 \\
3 1.72 \\
4 1.83\\
5 2.01\\
6 2.88 \\
7 5.04 \\
8 9.91 \\
9 12.34 \\
10 15.15 \\
    };
    \addlegendentry{\textcolor{black}{Ours}}
	\end{axis}
\end{tikzpicture}

%% file: sec/6_conclusion.tex
\section{Conclusion}
\label{sec:conclusion}

We introduce \textbf{\methodname}, the first unsupervised rotation-invariant dense correspondence method that unifies rigid and non-rigid matching. At its core is the novel feature extractor \textbf{\netname}, that enables end-to-end SO(3)-invariant feature learning on surfaces. 
The integration with orientation-aware CFMaps further ensures robustness to intrinsic shape symmetries. 
Our method takes a significant step toward shape matching methods that are applicable to real-world data and  achieves new SoTA across challenging non-rigid matching scenarios, including arbitrary poses, non-isometric deformations, partiality, non-manifold structures, and noise perturbations, thereby validating the feasibility and full potential of a genuinely data-driven non-rigid matching paradigm. 
As promising future work, we will explore matching methods that can directly consume the learned vector-valued features (avoiding feature flattening). We hope our work can inspire further advancements in rotation-invariant and data-driven 3D shape understanding.

\section*{Acknowledgment}
We thank Vladimir Golkov, Thomas Dag\`es, Viktoria Ehm, and Linus H\"arenstam-Nielsen for proofreading the early version of the manuscript and providing valuable feedback, which significantly improved the work. We are also grateful to Nicolas Donati and Simone Melzi for the fruitful and inspiring discussions.
We express our gratitude to the respective funding agencies. MG, SH, RM, and DC are supported by the ERC Advanced Grant SIMULACRON and the Munich Center for Machine Learning. MG also acknowledges support from the Bavarian Californian Technology Center. 
CD and LG gratefully acknowledge support from the Toyota Research Institute University 2.0 Program, a Vannevar Bush Faculty Fellowship, and a gift from the Flexiv Corporation. Additionally, CD is supported by a Tayebati Postdoctoral Fellowship.

%% file: sec/X_suppl.tex
\clearpage
\setcounter{page}{1}
\maketitlesupplementary

\section{\netname: Discussion and Proof}
\label{app: \netname}
\paragraph{VN-Gradient Layer.}
Following the notation in~\cite{sharp2022diffusionnet}, we introduce another formulation of our equivariant gradient layer. 
We first compute the spatial gradient $\vz \in \bbC^{n \times 3}$ of the diffused feature $\vh \in \bbR^{n \times 3}$ as $\vz = \cG(\vh)$, where $\cG$ is the intrinsic spatial gradient operator defined on the shape surface.
Then at each (single) vertex $v$, the local gradients of all feature channels are stacked to form $\vw \in \bbC^{c \times 3}$ and obtain real-valued features $\vf \in \bbR^{c}$ as:

\begin{equation}
\label{eq: feat_grad_math}
    \vf \coloneq \re (\diag ( \overline{\vw} (\mA \vw)^T ))
\end{equation}

where $\mA$ is a learnable square $\bbC^{c \times c}$ matrix, $\diag(\cdot)$ extracts the diagonal of an input (square) matrix and taking the real part by $\re(\cdot)$ after a complex conjugate $\overline{\vw}$ is a notational convenience for dot products between a pair of 2D vectors. 
The $i^{\mathrm{th}}$ entry of the output at vertex $v$ is given by the inner product
$\vf^i = \re ( \sum_{j=1}^c \langle  \overline{\vw}^i ,  \mA_{ij} \vw^{j} \rangle )$ with vectors $\overline{\vw}^i, \vw^j \in \bbC^{1 \times 3}$.

Note that Eq.~\ref{eq: feat_grad_math} is the same as Eq.~\ref{eq:feat_grad_raw} in the main paper by using a more compact notation, which will help to prove Theorem~\ref{lemma:feat_grad_raw_invariance} as shown below.
\paragraph{Proof of Theorem~\ref{lemma:feat_grad_raw_invariance}.}
\begin{proof}
    
Given a rotation matrix $\mR \in SO(3)$ applied on the input $\vw$, we leverage the property of rotation group and linear algebra tricks.

\begin{align}
\label{eq:proof_spatial_gradient}
    \vf^i (\vw \mR) &= \re ( \sum_{j=1}^c \langle \overline{\vw^i \mR},  \mA_{ij} (\vw^j\mR) \rangle) \\
    &= \re ( \sum_{j=1}^c \langle \overline{\vw}^i \mR,  \mA_{ij} \vw^j\mR \rangle) \\
    &= \re ( \sum_{j=1}^c \langle \overline{\vw}^i \mR  \mR^T,  \mA_{ij} \vw^{j} \rangle) \\
    &= \re ( \sum_{j=1}^c \langle \overline{\vw}^i, \mA_{ij} \vw^{j} \rangle) \\
    &= \vf^i (\vw) 
\end{align}

Hence, Eq.~\ref{eq: feat_grad_math} and $\vf$ is SO(3)-invariant.
\end{proof}

Note that applying $\tanh$ activation on invariant features does not change the invariance. Therefore, $\vg \coloneq \tanh (\vf)$ is also SO(3)-invariant.

\paragraph{Proof of Theorem~\ref{lemma:network_equivariance}.}
\begin{proof}
We first prove the equivariance of the \netname~block, which boils down to show that Eq.~\ref{eq:vn-diffnet_block} is equivariant. Since $\VNMLP$ is equivariant~\cite{deng2021vn}, it remains to show that the network input $[\vu, \vh, \ve]$ is equivariant. This is a direct consequence of the equivariance of heat diffusion $\cH_t(\cdot)$ and Theorem~\ref{lemma:feat_grad_raw_invariance}.
More specifically, if the input feature $\vu$ is rotated by $\mR \in SO(3)$, the corresponding diffusion feature $\vh$ and gradient feature $\ve$ will become $\vh \mR$ and $\ve \mR$ respectively. Hence their concatenation leads to $[\vu \mR, \vh \mR, \ve \mR] = [\vu, \vh, \ve] \mR$. Therefore Eq.~\ref{eq:vn-diffnet_block} is equivariant.

Based on the result above, we now prove the invariance of the whole \netname. As illustrated in Fig.~\ref{fig:network}, the \netname~is a linear composition of a chain of equivariant layers, followed by a $\VNINV$ layer and $\VNLIN$ layer. It is trivial to observe that:

\begin{itemize}
    \item Composition of equivariance functions keeps the equivariance.
    \item Composition of equivariant function with invariant function leads to invariance.
\end{itemize}
Hence, the whole \netname~is SO(3)-invariant.
\end{proof}

\section{Unsupervised SO(3)-Equivariant Matching}
\label{app: matching_method}

\subsection{Unsupervised Loss}
\label{app:loss_details}

Our unsupervised loss $\Ltotal = \Lstruct + \Lcouple + \Lcontrast$ guides the \netname~to learn features that are aware of the surface geometry and orientation, while being robust to shape artifacts such as partiality, non-manifoldness, and noise perturbation. Below, we provide the full formulation for each component.

\paragraph{Structural Loss.}
The structural loss ensures that the estimated functional maps exhibit desirable geometric properties. Following prior works \cite{roufosse2019unsupervised, donati2020deepGeoMaps}, we enforce near-orthogonality and bijectivity for the standard functional map $\mC$ and orthogonality for the complex functional map $\mQ$.
Although the orthogonality term penalizes deviation of both $\mC$ and $\mQ$ from being orthogonal, it corresponds to enforcing different properties: while the orthogonality of $\mC$ encourages local area-preservation, it is a necessary property for $\mQ$ to represent a valid push-forward \cite{donati2022CFMaps, donati2022DeepCFMaps}.

\begin{align}
    \Lorth(\mCxy) &= \| \mCxy^T \mCxy - \mI \| \\
    \Lorth(\mQxy) &= \| \mQxy^T \mQxy - \mI \| 
\end{align}

The bijectivity term enforces that the map from $\cX$ to $\cY$ and back to $\cX$ is the identity map, promoting cycle-consistency.

\begin{equation}
    \Lbij(\mCxy, \mCyx) = \| \mCxy \mCyx - \mI \|
\end{equation}

\noindent Our total structural loss, applied bi-directionally, is:
\begin{align}
\label{eq:loss_struct_full}
\Lstruct = &\lambda_{1} \Lorth(\mCxy) + \lambda_{1} \Lorth(\mCyx) \\
 + & \lambda_{2} \Lorth(\mQxy) + \lambda_{2} \Lorth(\mQyx) \nonumber \\
 + & \lambda_{3} \Lbij(\mCxy, \mCyx) + \lambda_{3} \Lbij(\mCyx, \mCxy) \nonumber
\end{align}
where $\lambda_{i}$ are Langragian weights.

\paragraph{Coupling Loss.}
The coupling loss ensures mutual consistency between the three map representations derived from the invariant features: the soft pointwise map $\mPi$, the functional map $\mC$, and the complex functional map $\mQ$.
Similar to previous work \cite{ren2021discrete, cao2023unsupervised, attaiki2023clover, sun2023spatially}, we couple the soft pointwise map $\mPiyx$ with $\mCxy$ by converting the pointwise map to its corresponding functional map $\mChxy = \mPhiy^{\dagger} \mPiyx \mPhix$:

\begin{equation}
    \Lpc = \| \mCxy - \mChxy \|_F^2
\end{equation}

\noindent Additionally, we propose to couple $\mPiyx$ with the orientation-aware complex functional map $\mQxy$. This is achieved by first converting $\mPiyx$ to $\mChxy$ and then converting $\mChxy$ to its complex functional map equivalent, $\mQhxy$, which can be done in closed-form using singular value decomposition (SVD) \cite[Appendix F]{donati2022CFMaps}:
\begin{equation}
    \Lpq = \| \mQxy - \mQhxy \|_F^2
\end{equation}
    
\noindent The total coupling loss is:
\begin{equation}
\label{eq:loss_couple_full}
 \Lcouple = \lambda_{4} (\Lpc^{\cX\cY} + \Lpc^{\cY\cX}) + \lambda_{5} (\Lpq^{\cX\cY} + \Lpq^{\cY\cX})
\end{equation}

We empirically found that directly coupling $\mC$ and $\mQ$ jeopardizes training convergence, especially during the early stages of training, when feature quality is still low, resulting in inferior performance (Tab.~\ref{tab:ablation-loss}). Our strategy of coupling through the more relaxed pointwise map provides a more stable regularization.

\paragraph{Contrastive Loss.}
This term encourages the learned invariant features to be distinct across the surface by ensuring bijective functional self-maps, $\mChxx$ and $\mChyy$ .
$$
\Lcontrast = \lambda_6 \left( \| \mChxx - \mI \|_F^2 + \| \mChyy - \mI \|_F^2 \right)
$$
where $\mChxx = \mPhix^{\dagger} \mPixx \mPhix$ and $\mPixx=\mathrm{Softmax}(\mFx\mFx^T/\tau)$ \cite{cao2024revisiting}, and similar for shape $\cY$.

\subsection{Orientation-Preserving Isometry}

The functional map $\mC$ is orientation-preserving, analogous to \cite{donati2022CFMaps}. Here, we provide the full proof for the ease of the reader.

\begin{theorem}
\label{lemma:compatibility_CQ} 
Let the functional map $\mC$ (possibly infinite-dimensional) estimated from features be an isometry and $\mQ$ the complex functional map estimated by the gradient of the features. Then the map $\mC$ is an orientation-preserving isometry.
\end{theorem}

\begin{proof}
Idea: we aim to prove that $\mC$ and $\mQ$ satisfy the following equation (cf. Eq.10 in~\cite{donati2022DeepCFMaps}):
\begin{equation}
\label{eq: relation_CQ}
    \langle \mX, \nabla_{\cX} (f \circ \varphi ) \rangle_{T_p\cX} = \langle \mQ(\mX), \nabla_{\cY} f \rangle_{T_{\varphi(p)}\cY}
\end{equation}

where $\varphi: \cX \rightarrow \cY$ is the isometric pointwise map, and $\mC: \cL^2(\cY) \rightarrow \cL^2(\cX)$ is the corresponding pullback, which transfers $\rmSO{3}$-invariant features, namely $\mC(\mFy) = \mFx$. 
Furthermore, the complex functional map, also known as the pushforward $\mQ: T\cX \rightarrow T\cY$ transfers the gradient of the invariant features $\mQ(\nabla_{\cX} \mFx) = \nabla_{\cY} \mFy$ and is complex-linear.
Given an arbitrary feature function $f: \cY \rightarrow \bbR$, a tangent vector field $\mX \in T\cX$ and a point $p \in \cX$,
we have the following equalities:
\begin{align}
\label{eq: proof_method_orientation}
    &\langle \mQ(\nabla_{\cX} \mFx), \nabla_{\cY} f \rangle_{T{\cY}} = \langle \nabla_{\cY} \mFy, \nabla_{\cY} f \rangle_{T{\cY}} \\
    & = \langle d\varphi^{-1}(\nabla_{\cY} \mFy),  d\varphi^{-1}(\nabla_{\cY} f) \rangle_{T{\cX}} \\
    & =  \langle \nabla_{\cX} (\mC \mFy), \nabla_{\cX} (\mC f) \rangle_{T{\cX}} \\
    & =  \langle \nabla_{\cX} \mFx, \nabla_{\cX} (\mC f) \rangle_{T{\cX}}
\end{align}

The first equality follows by applying the definition $\mQ(\nabla_{\cX} \mFx) = \nabla_{\cY} \mFy$.
The second equality holds since the metric is preserved by the pullback $d\varphi^{-1}$. (This assumes the invertibility of $\varphi$, which means exact correspondences.)
The third equality utilizes the fact that the pullback commutes with the gradient.
And the final equality follows due to the definition $\mC(\mFy) = \mFx$.

Let $\mX \coloneq \nabla_{\cX} \mFx $, we arrive at the Eq.~\ref{eq: relation_CQ} that we initially want to prove.
According to Theorem 2 in~\cite{donati2022DeepCFMaps}, the functional map $\mC$ is an orientation-preserving isometry.
\end{proof}
In particular, this theorem also holds for invariant (resp. equivariant) features estimated by \netname.

\section{Experimental Details}

\begin{table*}
\hspace{-1.cm}
\centering
    \begin{tabular}{lccccccc}
        \toprule
        {\small{\hspace{-0.2cm} }}& Ours & \ulrssm & \duofm & \consistfm  & \sms  & \sfm  & \hybridfm \\\cline{1-1}
        \hline
        \#params(k, $\downarrow$)&  \textbf{327.7} &  510.3 &  477.8 &   477.8 &  1,266.2 &  1,848.4 &  494.3 \\
        runtime(s, $\downarrow$) &  0.065 &  \textbf{0.023} &  0.071 &  0.055 &  15.311 &  0.171 &  0.043 \\
        mGeoErr ($\downarrow$) &  \textbf{4.6} &  24.8 &  25.2 &  9.1 &  57.6 &  9.9 &  26.4 \\
        \bottomrule
    \end{tabular}
\caption{\textbf{Comparison of runtime, \#params and mGeoErr on SMAL.}}
\label{tab:runtime}
\end{table*}

\begin{figure*}
    \centering
    \hspace{-3em}
    \begin{tabular}{cc}
        \includegraphics[width=0.5\linewidth]{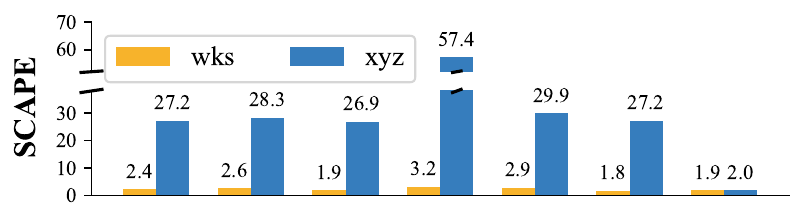}  & \includegraphics[width=0.5\linewidth]{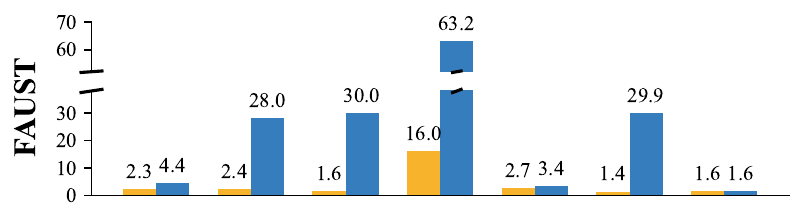}  \\
        \includegraphics[width=0.5\linewidth]{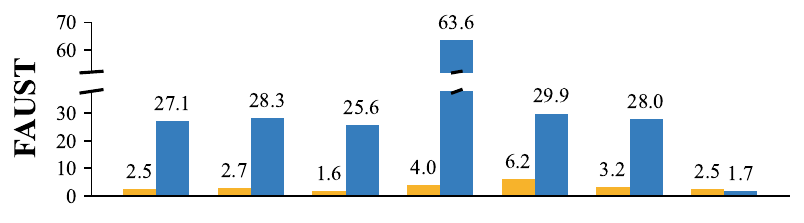} & \includegraphics[width=0.5\linewidth]{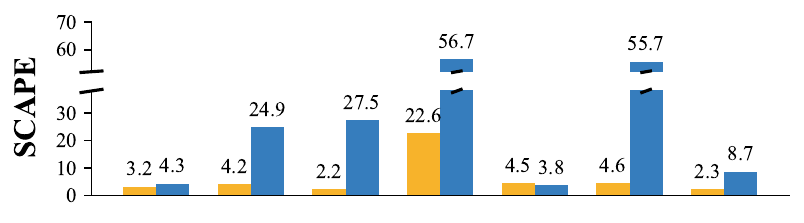} \\
        \includegraphics[width=0.5\linewidth]{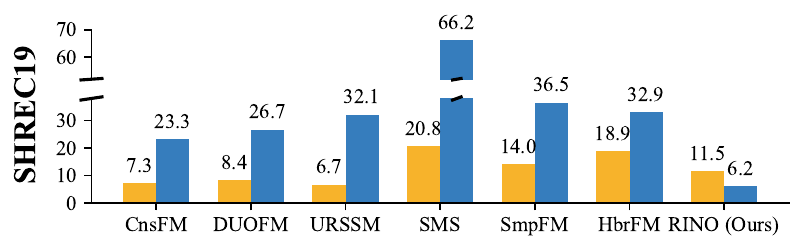} & \includegraphics[width=0.5\linewidth]{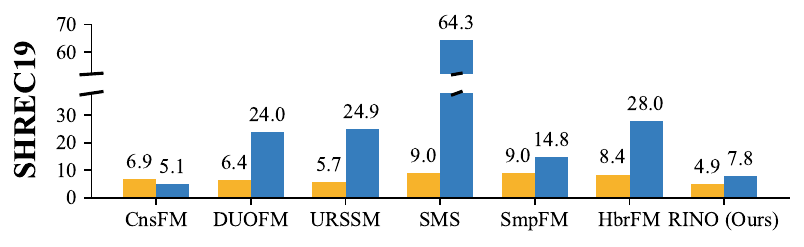} 
    \end{tabular}
    \caption{\textbf{MGeoErr($\downarrow$) on SCAPE, FAUST, SCHREC19.} (Left) trained on SCAPE; (Right) trained on FAUST. Y-axis dataset labels indicate the test dataset. MGeoErr is annotated on top of each bar. All baselines significantly deteriorate when changing from \wks~to \xyz, whereas ours is more robust to these changes. To our best knowledge, \methodname~is the first method, which learns competitive features and correspondences directly from raw 3D geometry in a purely data-driven manner.}
    \label{fig:scape_bar}
\end{figure*}

\paragraph{Implementation details.}
Our \netname~takes the xyz coordinate of the mass-centered 3D shape as input.
To facilitate a fair comparison, we chose the internal feature dimension in our \netname~block as $c=42$ such that the flattened feature has $126(=42 \times 3)$ dimensions, which is similar to the original \diffusionnet, whose hidden feature dimension is 128.
Note that our network has only $327,714$ training parameters, which is approximately $2/3$ of the \diffusionnet~($510,336$ trainable parameters), owing to the vector neuron representation.
The dimension of the final output features is set to $256$, same as baselines.
We chose $200$ eigenbases of the LBO and $30$ eigenbases of the connection Laplacian for computing $\mC$ and $\mQ$ respectively.
All hyper-parameters $\lambda$ in the total loss are set as $1.0$, except $\lambda_{2}$ and $\lambda_{5}$ to be $0.1$.

\paragraph{Datasets.} 
To study the performance under near-isometric deformation, we use
SCAPE~\cite{anguelov2005scape}, FAUST~\cite{bogo2014faust} and SHREC19~\cite{melzi2019shrec}.
The SCAPE dataset contains a single person in 71 different poses, where the train/test split is 51/20.
The FAUST dataset contains 10 persons in 10 different poses, leading to a total of 100 shapes, which are split into 80/20 for training and testing.
The challenging SHREC19 dataset contains 44 human shapes with different mesh connectivity and is only used as a test set.

For non-isometric cases, we consider the SMAL~\cite{Zuffi:CVPR:2017} and DT4D~\cite{magnet2022smooth} datasets. SMAL contains 49 four-legged animal shapes of 8 species, from which 5 are used for training and 3 unseen species for testing, resulting in a 29/20 split~\cite{donati2022DeepCFMaps}. 
The DT4D dataset contains challenging humanoid shapes with large non-isometric deformation across different shape classes. Following~\cite{donati2022DeepCFMaps, cao2023unsupervised}, we employ 9 different classes of humanoid shapes, resulting in a split of 198/95 shapes.

For real-world scan, we utilize the FAUSTSCAN~\cite{bogo2014faust} training dataset, which is the raw scan to create the registered FAUST dataset.  The FAUSTSCAN training dataset contains a total of 100 shapes with 10 people in 10 different poses, which is split in the same way for train/test as FAUST. Since it is real-world captured human scans, it has the typical artifacts of missing parts due to occlusion or lack of sensor coverage and scan noise. We subsample the shapes to a resolution of $\!\sim\!$~5k vertices for all evaluated methods.

We further stress test our method using a customized non-manifold version of FAUST, which is a dedicated non-manifold test set derived from the original test partition (20 test shapes), where half the shapes are decimated (\#vertex reduced from 5k to 3k). 

For partial shape matching, we employ SHREC16-Partiality \cite{cosmo2016shrec, ehm2024partial}. It has two sub-datasets, namely CUTS and HOLES, representing different partial modalities. It contains 120 training and 153 testing shapes for CUTS; 80 training and 200 testing shapes for HOLES, respectively.

\paragraph{Computational complexity.}

The performance metrics, including the number of trainable parameters and mean runtime on the SMAL test set, are detailed in Tab.~\ref{tab:runtime}. For a fair comparison, we employ the same computation unit with four Intel Xeon E5-2697 CPU and one NVIDIA TITAN Xp GPU  consistently for all methods.
Our model establishes a new frontier by being both significantly smaller (approximately 30\% fewer parameters than the next-best method) and substantially more accurate (half of the error of the closest competitor). 
While the current implementation demonstrates a marginal increase in runtime, it operates within the competitive range of all other methods, requiring less than 0.1 seconds to process a shape pair with over 5k vertices. 
This performance gap could be effectively mitigated through the development of a CUDA-accelerated implementation of the VN-layers.

\section{Additional Experiments}

\subsection{\xyz~versus \wks}

Existing methods, including \ulrssm~and \hybridfm, demonstrate a selective preference for shape descriptors. 
Specifically, they employ \xyz~coordinates for the SMAL and SHREC16-Partiality datasets, but utilize \wks~for others, such as SCAPE and FAUST. 
This differentiation is likely necessitated by the significant non-isometries inherent in SMAL and SHREC16-Partiality, which inherently reduce the stability of \wks. 
To rigorously analyze this effect, we calculated the mean values of the feature channels across semantically corresponding points for all shapes in both the SMAL and SCAPE datasets. 
The empirical evidence in Fig.~\ref{fig:stat_wks_xyz} confirms that \wks~possesses a higher variance within \textcolor{cPLOT3}{SMAL} than within \textcolor{cPLOT1}{SCAPE}, whereas the \xyz~coordinate shows greater consistency. 
The inherent ability of \xyz~to represent raw shape geometries makes it a compelling choice for purely data-driven 3D learning in both research and practical applications.

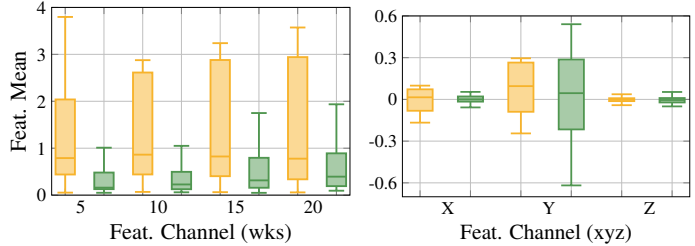
\begin{figure}
    \centering
    \begin{tabular}[th!]{cc}
    \hspace{-1.4cm}
    \resizebox{0.632\columnwidth}{!}{%
    \input{figs/wks/boxplot_wks}
    }
    &
    \hspace{-1cm}
    \resizebox{0.6\columnwidth}{!}{%
    \input{figs/wks/boxplot_xyz}
    }
    \vspace{-0.7em}
    \end{tabular}
    \caption{\textbf{Statistics of \wks~vs. \xyz.} The box plot illustrates the statistical distribution of the mean feature channel values at semantically corresponding points across the entire shape collections of the \textcolor{cPLOT3}{SMAL} and \textcolor{cPLOT1}{SCAPE} datasets. The methodology involved aggregating the channel-wise mean of features across corresponding points for every shape, yielding a feature matrix (with dimensions \#correspondences $\times$ \#feature channels), where each column's distribution is plotted as a box. The statistics shown for the 5th, 10th, 15th, and 20th \wks~channels demonstrate a pronounced difference in behavior between the datasets. Specifically, the \wks~statistics for \textcolor{cPLOT3}{SMAL} display a substantially larger variance than those for \textcolor{cPLOT1}{SCAPE}. This higher dispersion is a direct consequence of the non-isometries present in the SMAL dataset, which contrasts sharply with the stability observed in the \xyz~descriptor statistics.}
    \label{fig:stat_wks_xyz}
\end{figure}

\subsection{Near-Isometric Matching}
We compare performance on the near-isometric datasets SCAPE, FAUST and SHREC19 using both \xyz~and \wks~as input features. Specifically, we train on SCAPE and FAUST, and test across all three datasets, with quantitative results shown in Fig. \ref{fig:scape_bar}. 
While all methods perform well when utilizing \wks, ours is uniquely effective when relying solely on \xyz, underscoring its extraordinary robustness to input type and its ability to learn salient shape information directly from raw data. This represents, to our knowledge, the first instance where \xyz~has achieved competitive results, if not better, against \wks~for near-isometric shapes. The high performance of baselines using \wks~also suggests that the near-isometric matching task has been addressed very well, hinting that the community should shift focus to more challenging non-isometric and real-world cases.
The experiments also show that \methodname~generalized well across different datasets, suggesting that equivariance/invariance does not hinder generalization. 
Furthermore, the shapes in SHREC19 have all different discretizations, ours performs well in this challenging scenario under both \xyz~and \wks~as input features. Despite the discretization-dependent VN-EdgeConv operation at the beginning of \methodname, the subsequent \netname~block learn adaptive diffusion time accordingly, rendering the entire \netname~robust to varying discretizations.

\subsection{Matching Non-manifold Shapes}

\begin{figure}[t!]
    \centering
    \hspace*{-0.5cm}
    \includegraphics[trim={300 530 450 30}, clip, width=0.55\textwidth]{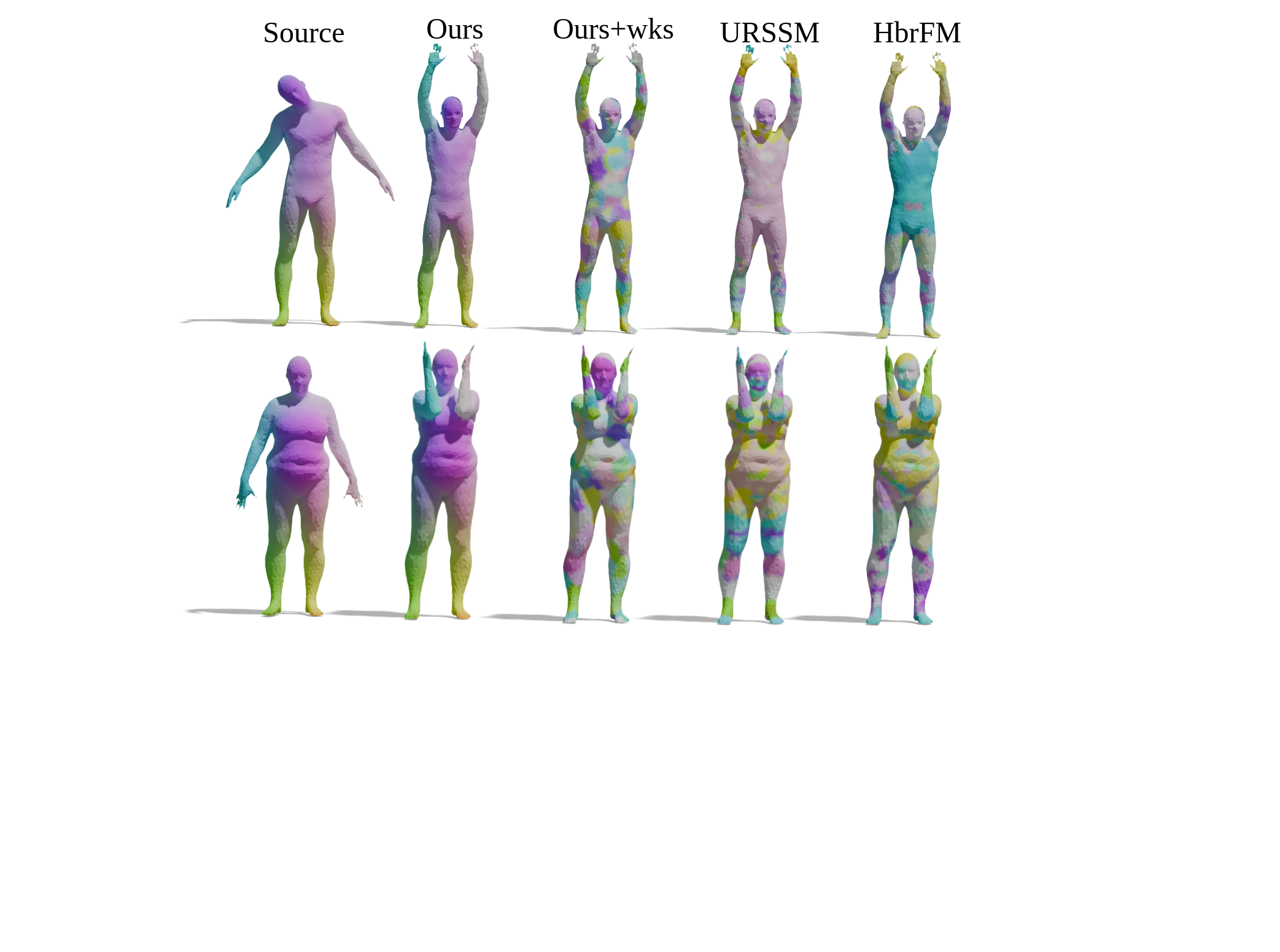}
    \caption{\textbf{Non-manifold shape matching.} Methods employing the \wks~as input exhibit a significant performance deterioration, even for SoTA methods such as \ulrssm~and \hybridfm, an observation that also applies to our proposed method (Ours+\wks). Fortunately, our SO(3)-invariance network allows for effective learning directly from the raw \xyz~coordinates (Ours). This approach offers robustness against the difficulties posed by non-manifoldness, bypassing the fragility observed when using \wks. }
    \label{fig:non-manifold}
\end{figure}

\begin{table}[t]
\centering
\begin{tabular}{lrr}
\toprule
& \textbf{FAUST} & \textbf{FAUST non-manifold}\\
\midrule
\ulrssm(\wks)  & 1.6 & 17.6  \\
\hybridfm(\wks)  & 1.4 & 16.6  \\
Ours(\wks) & 1.6 & 15.6 \\
Ours(\xyz) & 1.6 & 2.7 \\
\bottomrule
\end{tabular}
\caption{\textbf{MGeoErr($\downarrow$) on FAUST non-manifold shapes.}}
\label{tab:non-manifold}
\end{table}

\begin{figure*}[t!]
    \vspace{-0.5cm}
    \includegraphics[trim={200 280 50 270}, clip, width=1.1\textwidth]{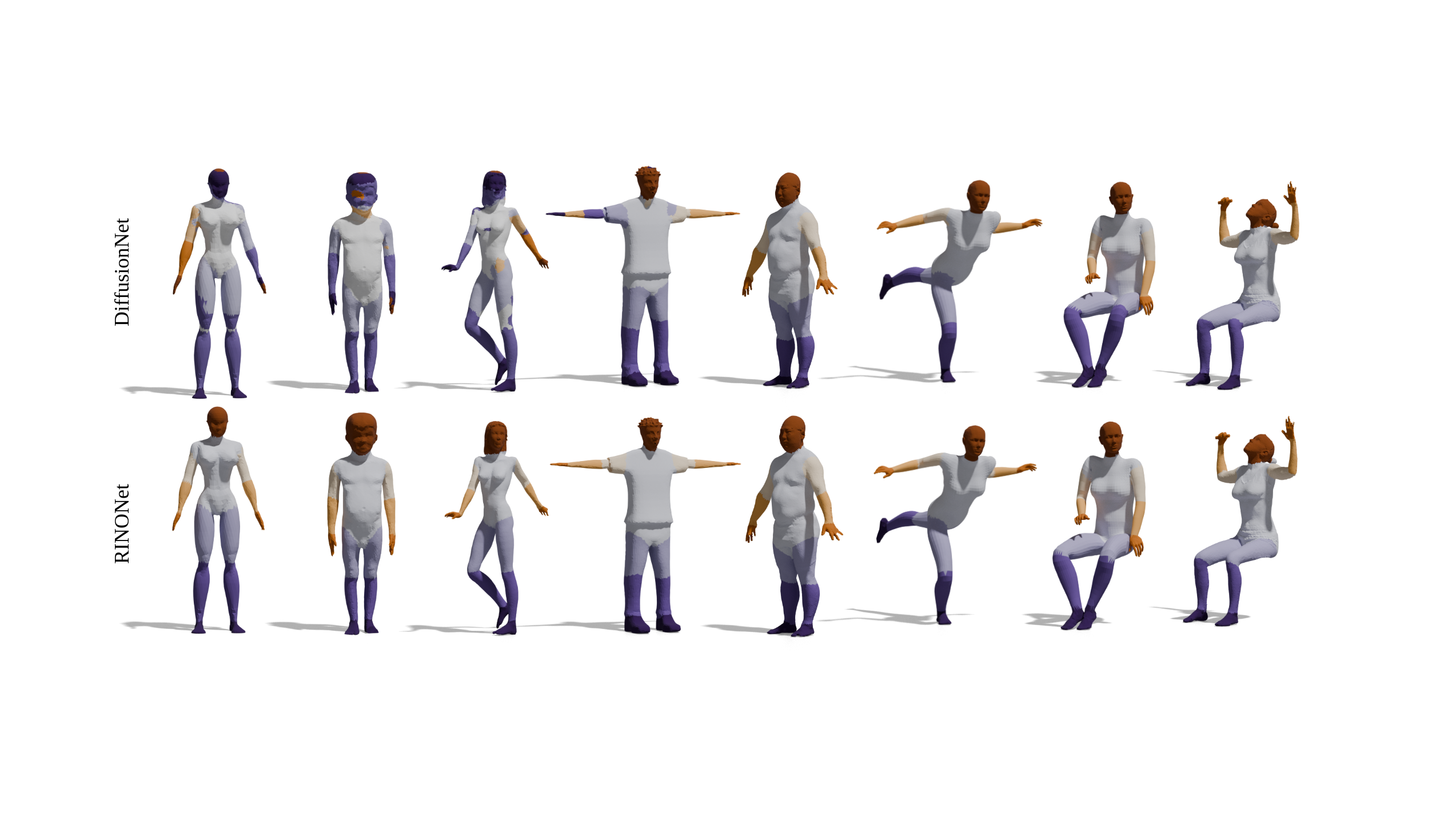}
    \caption{\textbf{Human Segmentation.} We show the segmentation results of \diffusionnet~and our \netname~on the composite dataset~\cite{maron2017cnn}. Our \netname~produces sharper and more semantically meaningful segments.}
    \label{fig:seg-supp}
    \vspace{18pt}
\end{figure*}

To underscore the advantage of directly processing raw geometry, we evaluate models pre-trained on FAUST against our customized FAUST non-manifold test set. 
Note that this setup violates the ``nice mesh" assumption required by intrinsic descriptors like \wks. Consequently, as shown in Tab.~\ref{tab:non-manifold}, methods relying on \wks~as input suffer a significant decline in matching performance. Crucially, our proposed method, when directly utilizing the raw \xyz~coordinates, exhibits remarkable robustness with only a negligible performance degradation. This result highlights the enhanced robustness conferred by \xyz~input when confronting the topological complexities of non-manifold and highly discrepant data. See Fig.~\ref{fig:non-manifold} for visual examples.

\subsection{Real-World Scans}

To evaluate all methods under real-world scans, we utilize the FAUSTSCAN training dataset, which contains typical artifacts of missing parts due to occlusion or lack of sensor coverage and scan noise.
We reduce scan vertices to 5k, due to the computation of LBO and its eigendecomposition, and no further preprocessing is performed to preserve the noise, holes and artifacts in the original scans.
The quantitative results are reported in Tab.~\ref{tab:non-isometry} (col. FSCAN). Ours performs all baselines regardless of input features and qualitative results are illustrated in Fig.~\ref{fig:faustscan}.
This strong performance suggests a big step towards shape matching methods for real-world applications.

\subsection{Human Segmentation}

\begin{table}[t]
\centering
\begin{tabular}{lrr}
\toprule
\textbf{Method}& \textbf{Human dataset~\cite{maron2017cnn}}\\
\midrule
\diffusionnet(hks, w/o aug)  & 91.3\%  \\
Ours(hks, w/o aug)  & 91.0\%  \\
\midrule
\diffusionnet(\xyz, w/ aug)  & 90.3\%  \\
\diffusionnet(\xyz, w/o aug)  & 83.6\%  \\
Ours(\xyz, w/ aug)  & 90.6\%  \\
Ours(\xyz, w/o aug)  & 90.1\%  \\
\bottomrule
\end{tabular}
\caption{\textbf{Segmentation accuracy ($\uparrow$) on human dataset~\cite{maron2017cnn}.}}
\label{tab:seg}
\end{table}

We compare against the SoTA \diffusionnet~\cite{sharp2022diffusionnet} across two feature inputs: \hks~and \xyz. The results in Tab.~\ref{tab:seg} show that our proposed method achieves comparable accuracy to \diffusionnet~when using \hks. The key distinction arises with raw \xyz~input: \netname~exhibits complete stability, maintaining its performance with or without SO(3)-augmentation during training, whereas \diffusionnet~suffers a substantial degradation when augmentation is removed. This robust, stable performance stems directly from the in-built SO(3)-invariance of our architecture, which simplifies the training pipeline and translates to enhanced training efficiency. In Fig.~\ref{fig:seg-supp} we present additional qualitative results.

\begin{figure}
    \centering
    \includegraphics[width=\linewidth]{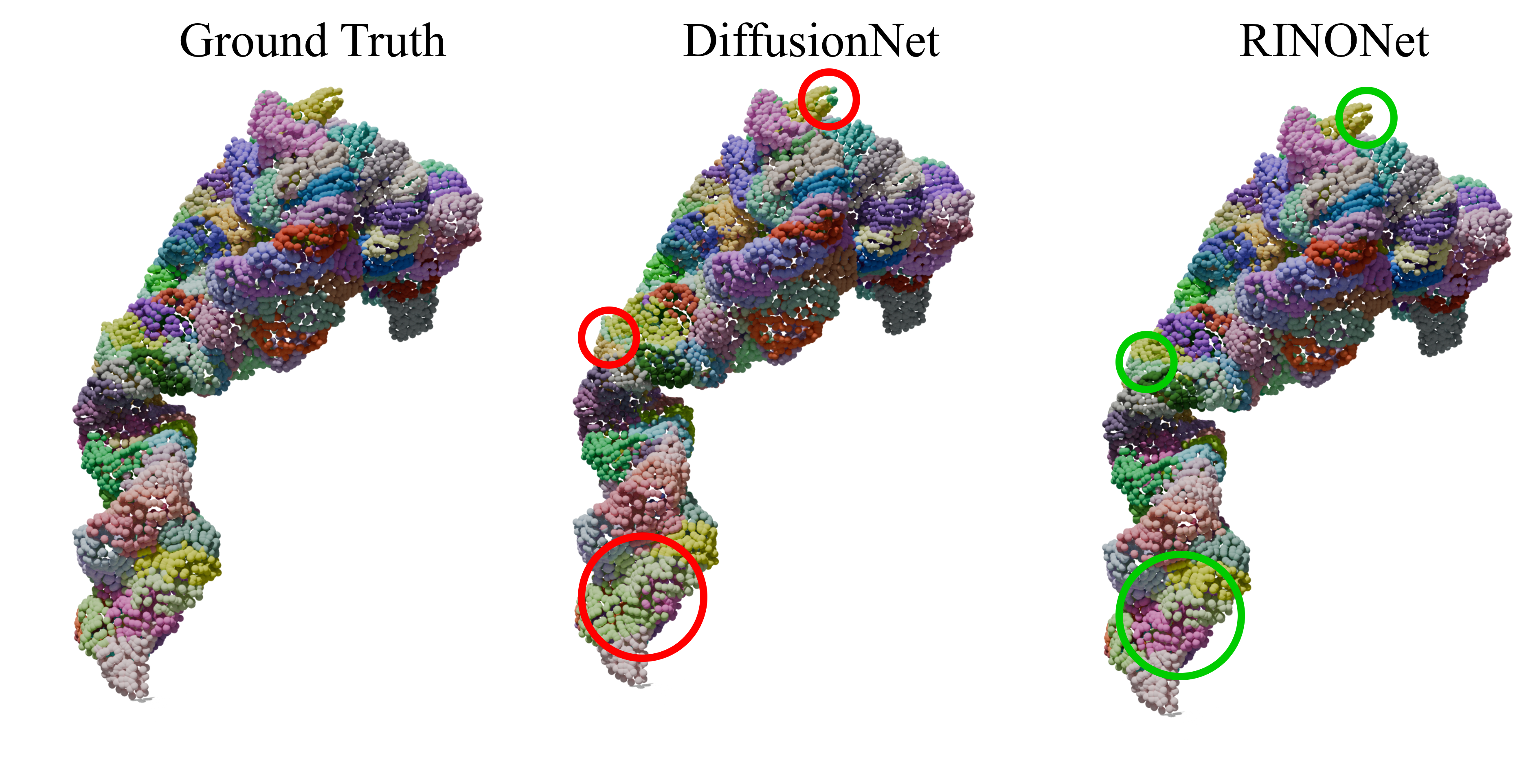}
    \caption{RNA segmentation}
    \label{fig:rna}
\end{figure}

\subsection{RNA Segmentation}
To test \methodname~on other diverse data modalities, we train it on RNA dataset~\cite{poulenard2019effective} with the same training protocol as DiffusionNet (i.e. the I/I setting). \netname~achieves $85.8\%$ accuracy, versus $72.1\%$ of DiffusionNet (cf. Fig.~\ref{fig:rna}).

\subsection{Rigid Registration}
As a proof-of-concept, we train \methodname~on the Stanford Bunny dataset~\cite{turk1994bunny}, predicting per-point SO(3)-invariant features. 
More specifically, we take two bunnies of different resolution and randomly rotate one of them. 
Then, a correspondence search is conducted in the feature space, and the rigid registration is computed using Procrustes Analysis. 
Our final mean geodesic error of correspondences is $2.4$, and the Chamfer distance is $0.018$. A visual result is shown in Fig.~\ref{fig:rigid_reg}.

\begin{figure}
    \centering
    \includegraphics[width=\linewidth]{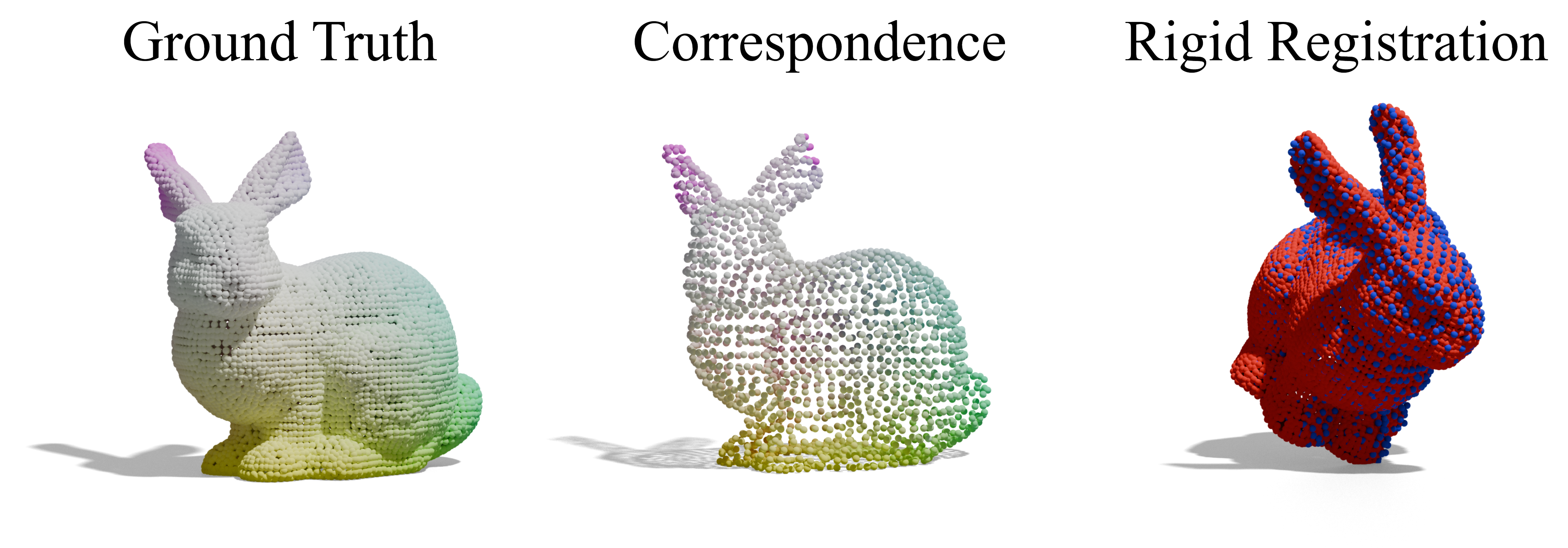}
    \caption{3D rigid registration on Stanford dataset (bunny)}
    \label{fig:rigid_reg}
\end{figure}

\begin{table}[t]
\centering
\begin{tabular}{llr}
\toprule
\textbf{No.} & \textbf{Ablation Setting} & \textbf{SMAL} \\
\midrule
1& w\textbackslash o $\Lstruct$  &  7.3 \\
2& w\textbackslash o $\Lcouple$   &  7.6 \\
3& w\textbackslash o $\Lcontrast$ &  5.7 \\
4& w\textbackslash o the Qmaps branch &   5.1\\
5& w\textbackslash o $\Lpq$  &  5.0 \\
6& Ours + coupling $\mC,\mQ$  &  26.9 \\
7& Ours & \textbf{4.6} \\
\bottomrule
\end{tabular}
\caption{\textbf{Ablation study on SMAL.}}
\label{tab:ablation-loss}
\end{table}

\subsection{Ablation}
We ablate the design of both our architecture and loss terms, and report the results on the challenging non-isometric SMAL dataset in Tab.~\ref{tab:ablation-loss}. 
To this end, we compare each row with the last row (our full model) to assess the importance of individual terms.
The performance degrades most when $\Lstruct$ or $\Lcouple$ is removed (first 2 rows), indicating the central role of structural property and consistency of different map representations. 
Interestingly, this finding echoes with previous works~\cite{roufosse2019unsupervised, ren2021discrete, cao2023unsupervised}.
Removing the $\Lcontrast$ term has a moderate impact on the final results, since it helps to learn more distinctive rotation-invariant features.
Looking at the fourth and fifth rows, we conclude that the integration of complex functional maps encourages the first-order consistency of rotation-invariant features, which in turn helps to improve the accuracy of estimated correspondences.
Finally, comparing the last two rows shows that excessive coupling can hinder learning performance.

To study the relative performance gain of invariant features and CFMaps, we ablate further in Tab.~\ref{tab:rino_vs_diffusion} by two experiments under the SO(3)/SO(3) setup. 
(i) We take baselines and swap in our \netname~ as the feature extractor and experiment on SMAL; \netname~ always improves over DiffusionNet, highlighting the benefit of SO(3)-invariant features. (cf. Fig.~\ref{fig:rinonet_diffusionnet_baseline})
(ii) We compare DiffusionNet and \netname~ without CFMaps, and \netname~ with CFMaps, training them on four different datasets, showing that the largest performance gain comes from SO(3)-invariant features, which can be further improved by CFMaps.

\begin{figure}[t!]
    \centering
    \includegraphics[width=1.02\linewidth,trim=13cm 9.8cm 16cm 20.2cm, clip]{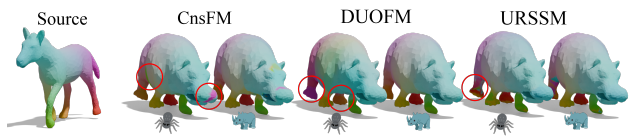}
    \vspace*{-1.5em}
    \captionlistentry{Description for List of Figures}
    \caption{Backbone comparison. Spider emoji denotes DiffusionNet and rhinoceros emoji denotes RINONet.}
    \label{fig:rinonet_diffusionnet_baseline}
\end{figure}

 \begin{table}[t!]
 
     \hspace*{-11pt}
     \footnotesize
     \begin{tabular}{clcccc}
     \hline
       & mGeoErr ($\downarrow$) & \duofm & \consistfm & \ulrssm & \hybridfm  \\ \cline{2-6}
       (i)   & {DiffusionNet} & 25.2 & 9.1 & 24.8 & 26.4  \\
         & {\netname~} & \textbf{5.7} & \textbf{7.3} & \textbf{5.1} & \textbf{7.4} \\
         \hline
           \multirow{4}{*}{(ii)}  &  & SCAPE & FAUST & SMAL & DT4D \\ \cline{2-6}
          & {DiffusionNet} & 26.9 & 30.0 & 24.8 & 59.3 \\
         & {\netname~} & \textbf{2.0} & 1.6 &  5.1 & 5.9 \\ 
         & {\netname~ + CFMaps} & \textbf{2.0} & \textbf{1.5} &  \textbf{4.6} & \textbf{5.3} \\
         \hline
     \end{tabular}
     \caption{Ablation: SO(3)-invariant features \& CFMaps}
     \vspace*{-20pt}
     \label{tab:rino_vs_diffusion}
 \end{table}

\begin{figure*}[b]
    \centering
     \includegraphics[trim={40 50 130 35}, clip, width=0.7\textwidth]{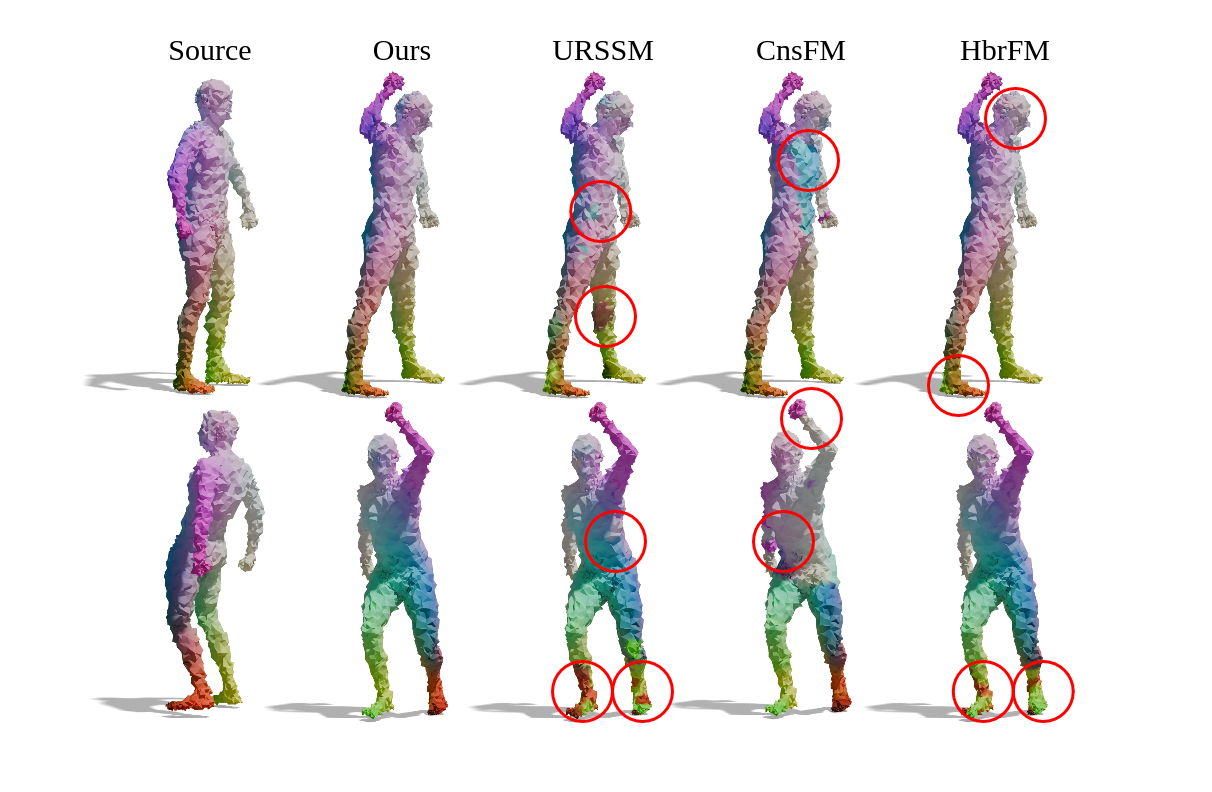}
    \caption{\textbf{Matching shapes with Gaussian perturbation.}}
    \label{fig:noise-vis}
\end{figure*}

\subsection{More Qualitative Results}
We show more matching results of \methodname~on the challenging noisy shapes in Fig.~\ref{fig:noise-vis}, and ours is the only method that can produce reliable correspondences under the noise level $\sigma=6\text{e-}3$, where baselines already start to produce spurious correspondences.

We also present additional qualitative results on the SHREC16-Partiality datasets. Here we show all partial shapes of all categories in Fig.~\ref{fig:cuts-cat}-\ref{fig:holes-michael}, and \methodname~manifests strong robustness under this challenging incomplete shape by directly consuming the raw 3D geometry. The highly accurate correspondences suggest remarkable quality of the learned rotation-invariant features.

\subsection{Failure Cases}

We notice that for highly non-isometric pairs, our estimated correspondences are sometimes locally non-smooth. As shown in Fig.~\ref{fig:failure}, \methodname~confuses when one humanoid misses completely his lower arm. This leads to wrong correspondences in the region of missing parts in the other shape.
We hypothesize that incorporating further intrinsic prior during training or post-processing could solve these artifacts.

\begin{figure}[t!]
    \centering
    \vspace*{-25em}
    \includegraphics[trim={600 550 300 200}, clip, width=0.5\textwidth]{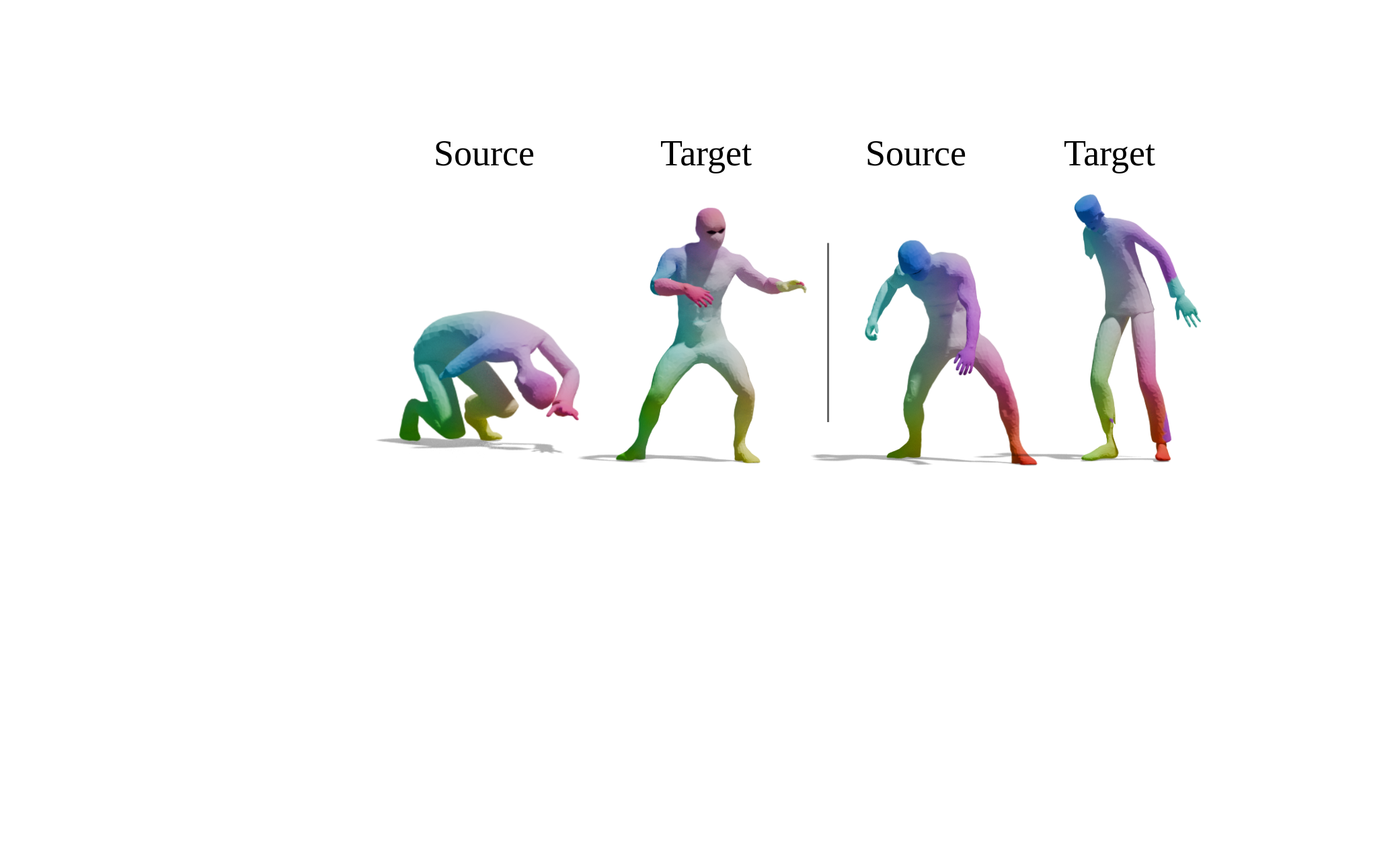}
    \caption{\textbf{Failure cases.}}
    \label{fig:failure}
\end{figure}

\begin{figure*}
    \centering
    \hspace{-0.5cm}
    \includegraphics[trim={400 530 450 30}, clip, width=1.0\textwidth]{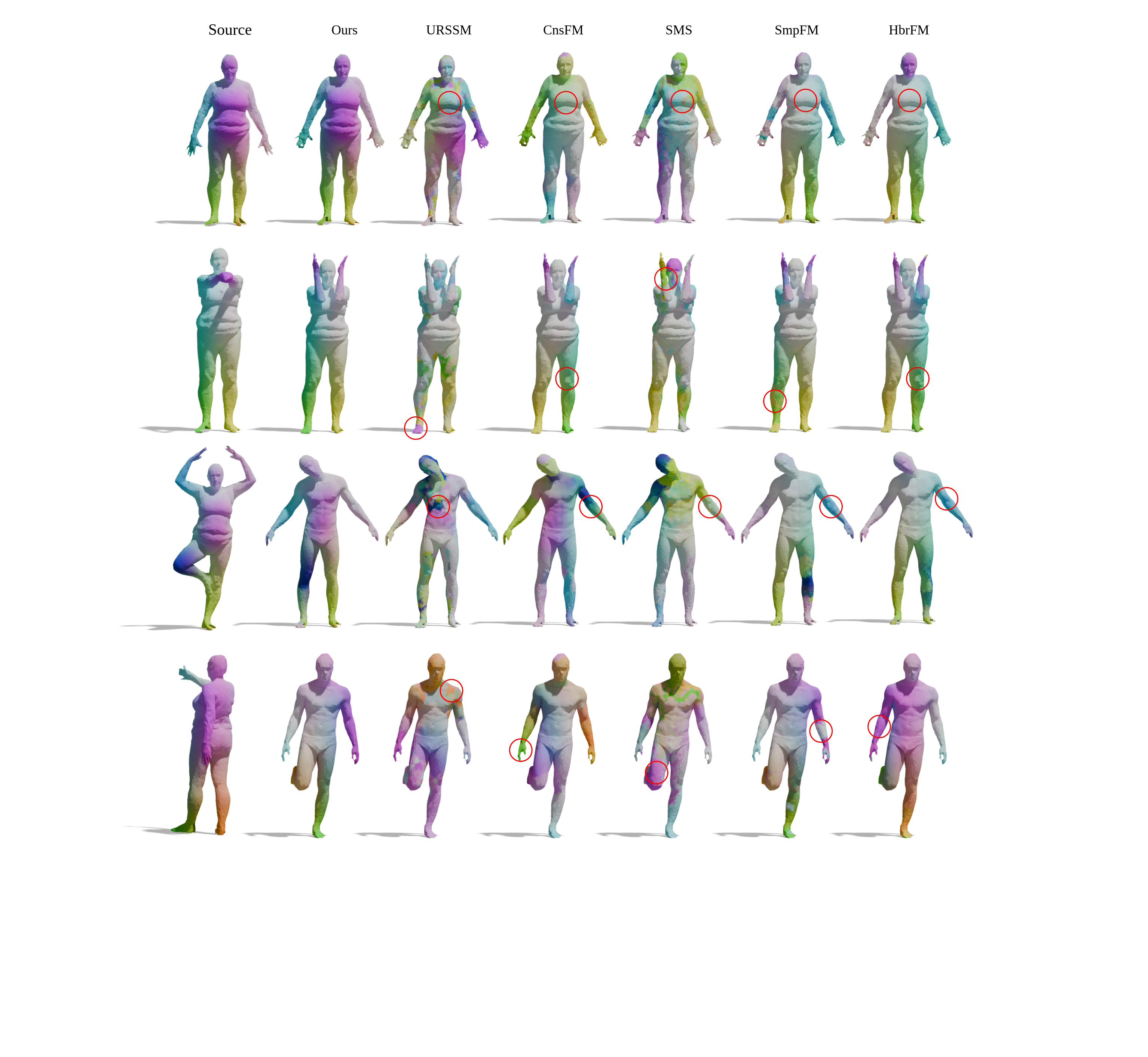}
    \caption{\textbf{Qualitative results on the real-world FAUSTSCAN.}}
    \label{fig:faustscan}
\end{figure*}

\begin{figure*}[b]
    \centering
    \hspace*{-4em}
    \includegraphics[trim={100 300 0 300}, clip, width=1.2\linewidth]{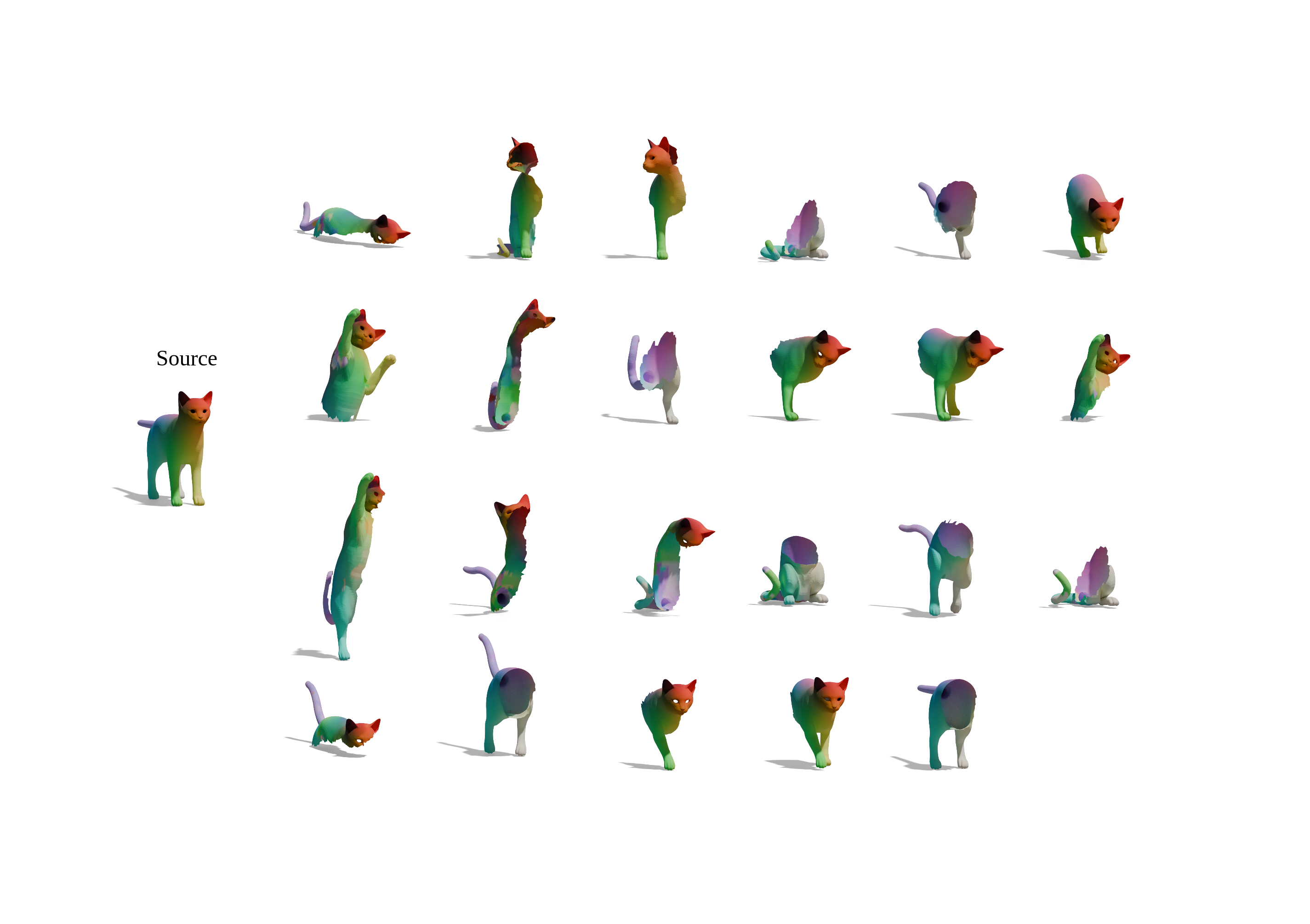}
    \caption{SHREC16-Partiality CUTS cat}
    \label{fig:cuts-cat}
\end{figure*}

\begin{figure*}
    \centering
    \hspace*{-4em}
    \includegraphics[trim={0 50 0 100}, clip, width=1.2\linewidth]{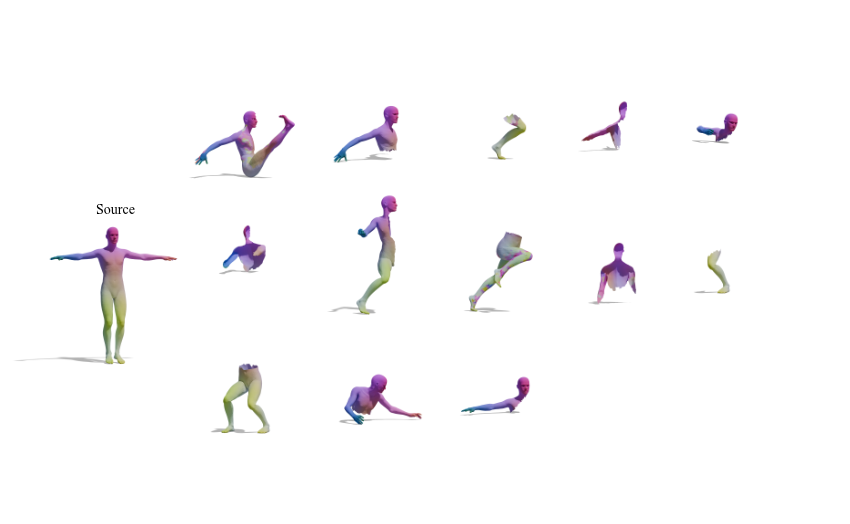}
    \caption{SHREC16-Partiality CUTS david}
    \label{fig:cuts-david}
\end{figure*}

\begin{figure*}
    \centering
    \hspace*{-1em}
    \includegraphics[width=0.95\linewidth]{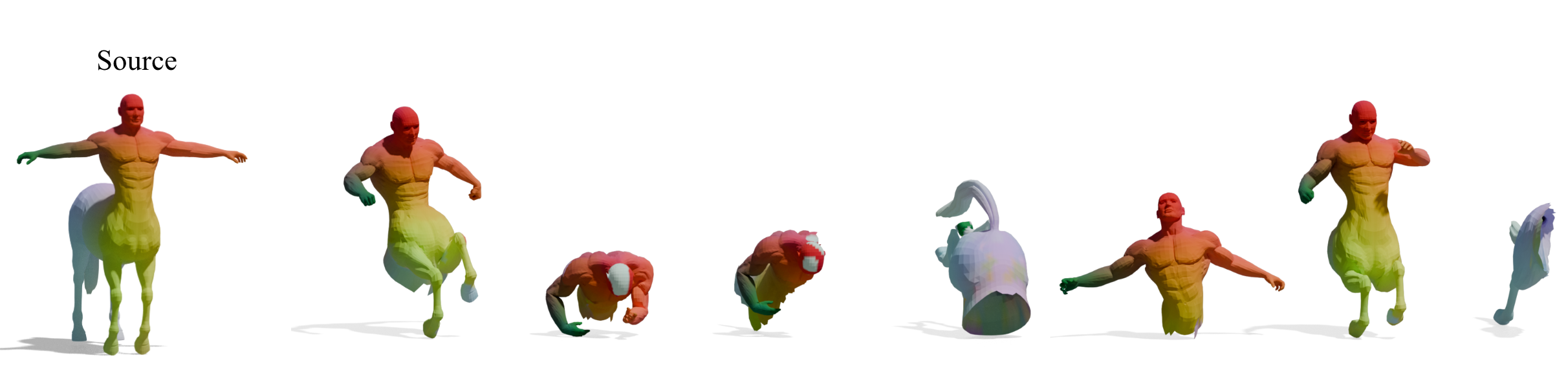}
    \caption{SHREC16-Partiality CUTS centaur}
    \label{fig:cuts-centaur}
\end{figure*}

\begin{figure*}
    \centering
    \includegraphics[width=\linewidth]{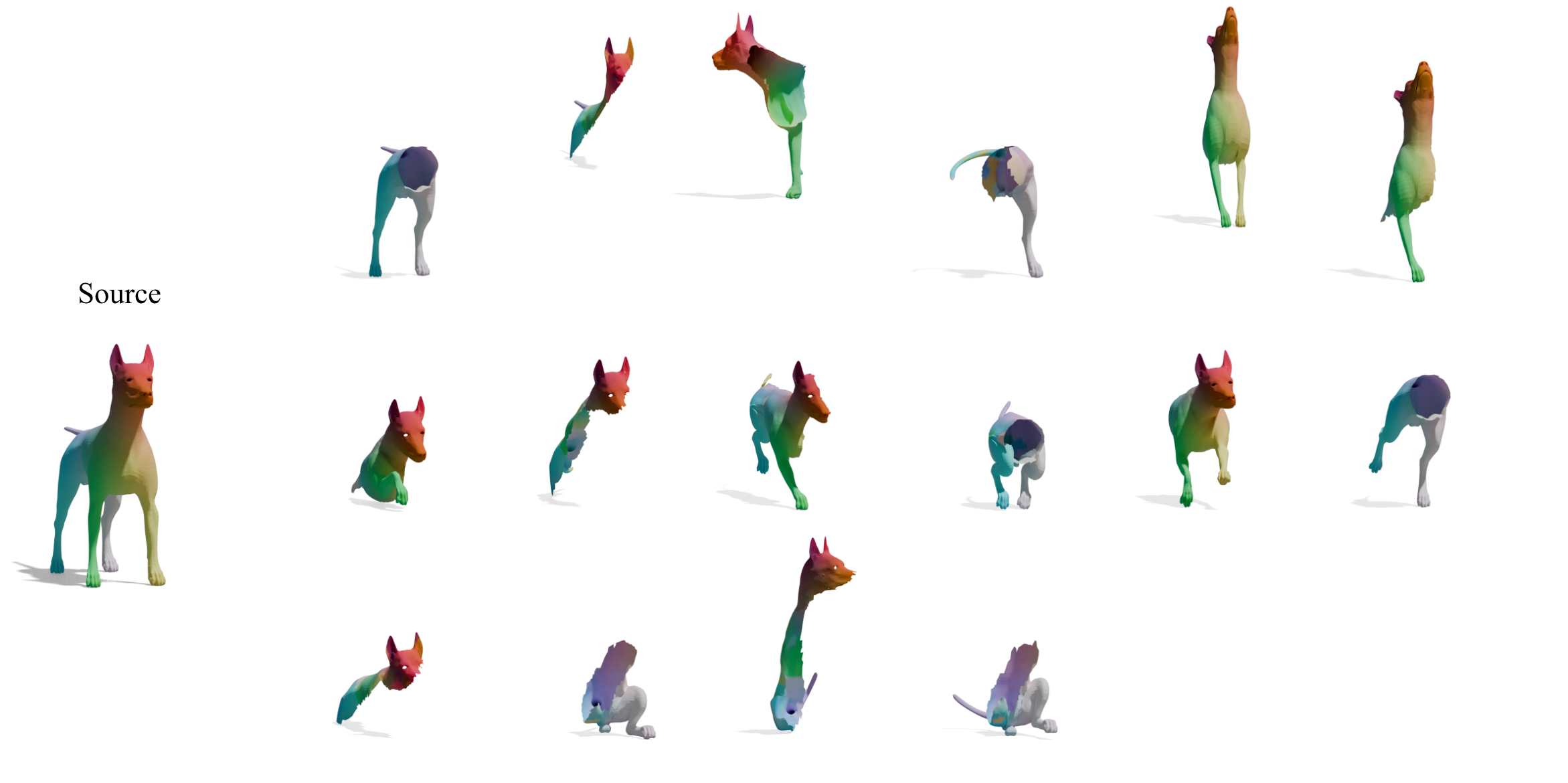}
    \caption{SHREC16-Partiality CUTS dog}
    \label{fig:cuts-dog}
\end{figure*}

\begin{figure*}
    \centering
    \includegraphics[width=\linewidth]{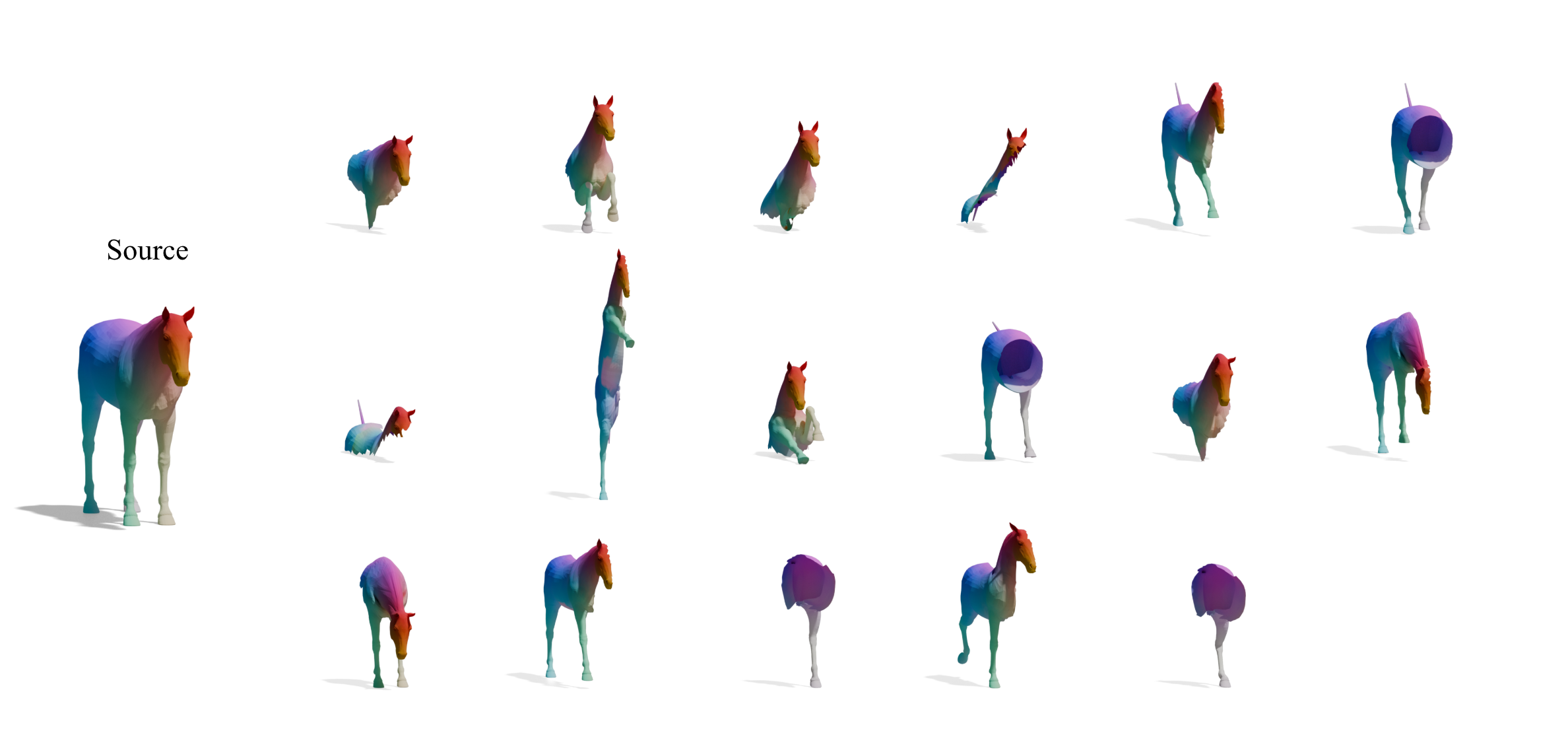}
    \caption{SHREC16-Partiality CUTS horse}
    \label{fig:cuts-horse}
\end{figure*}

\begin{figure*}
    \centering
    \includegraphics[width=\linewidth]{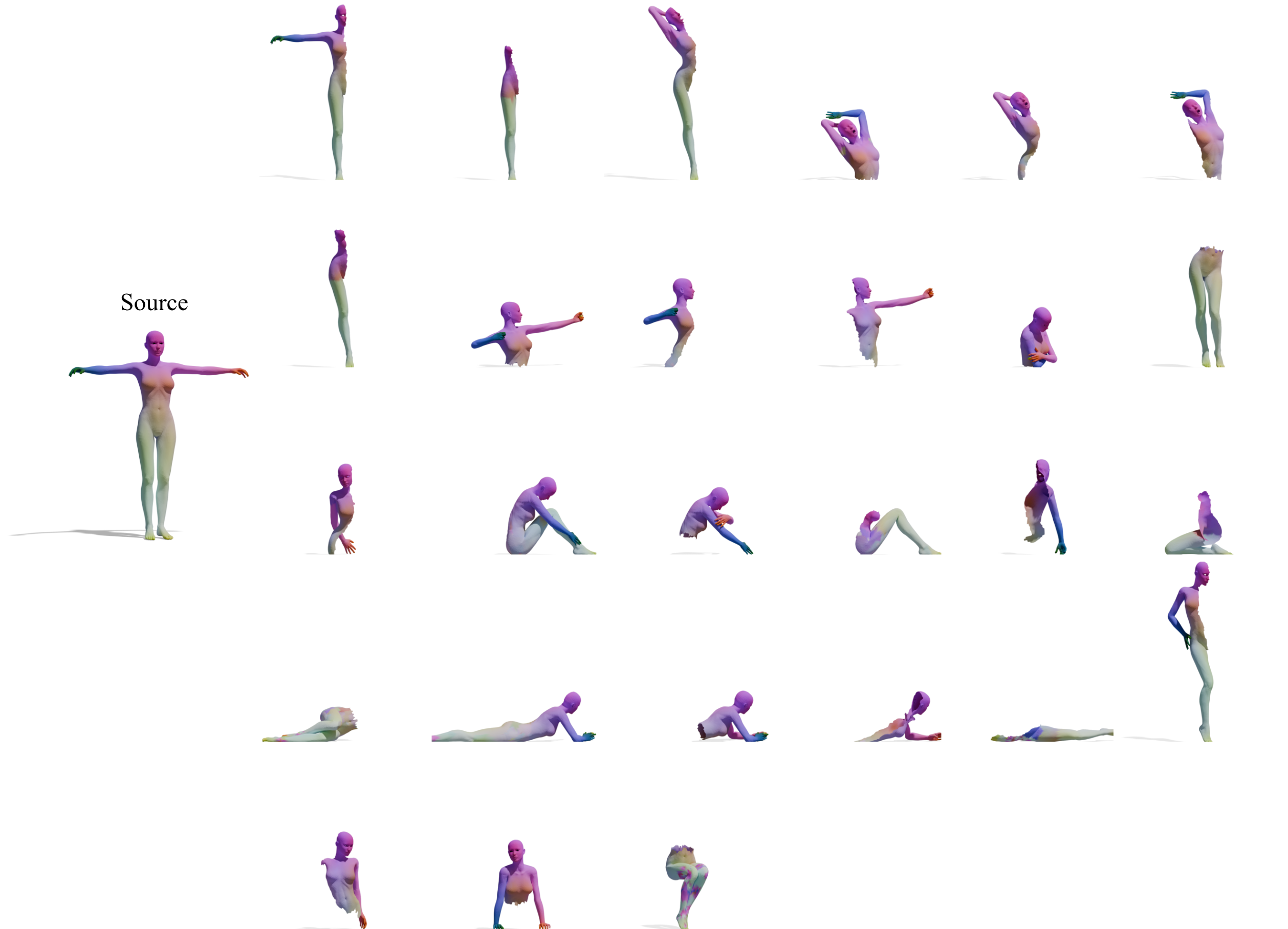}
    \caption{SHREC16-Partiality CUTS victoria}
    \label{fig:cuts-victoria}
\end{figure*}

\begin{figure*}
    \centering
    \vspace*{-2em}
    \hspace*{-3em}
    \includegraphics[width=1.1\linewidth]{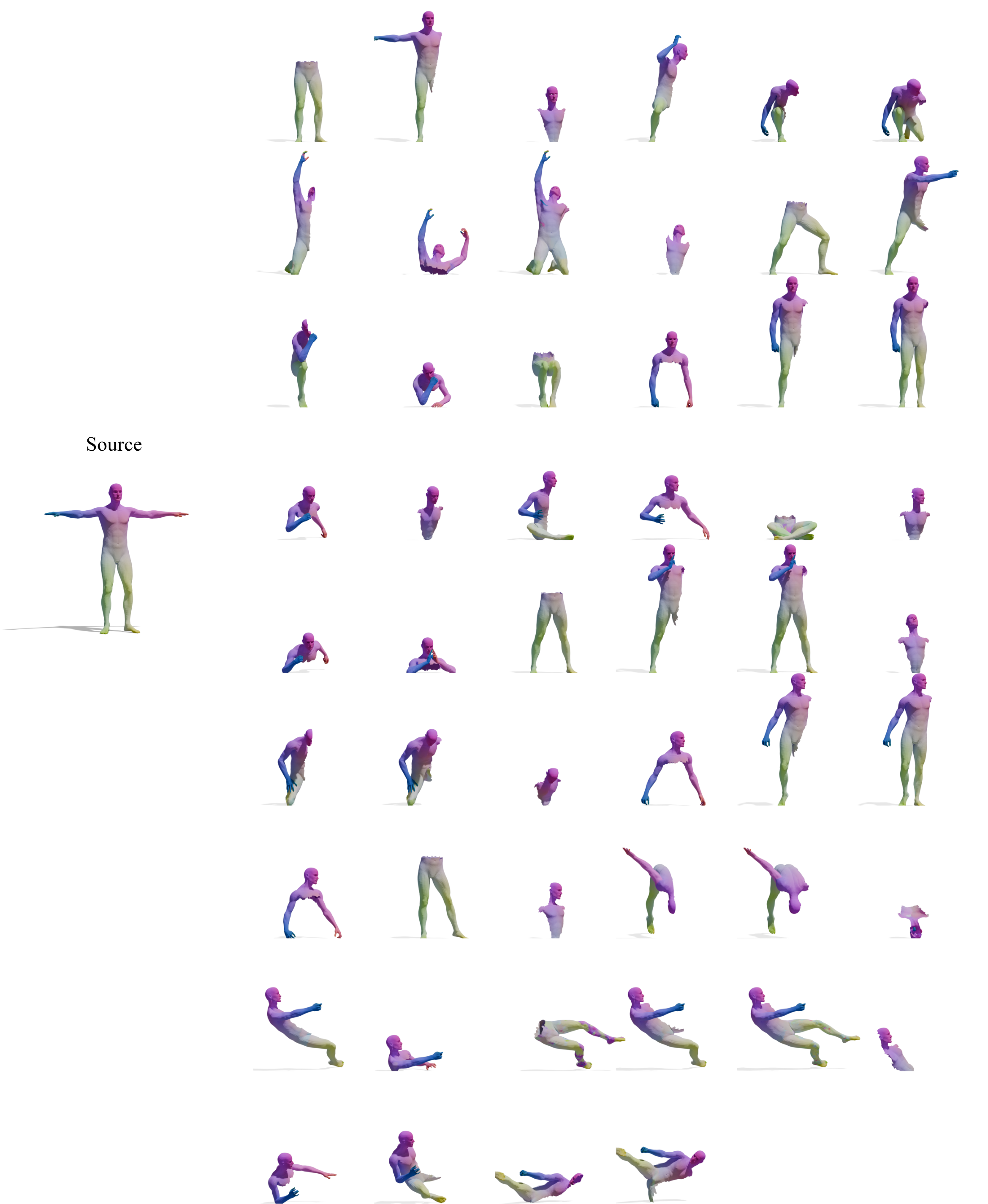}
    \caption{SHREC16-Partiality CUTS michael}
    \label{fig:cuts-michael}
\end{figure*}

\begin{figure*}
    \centering
    \hspace{-3em}
    \includegraphics[width=\linewidth]{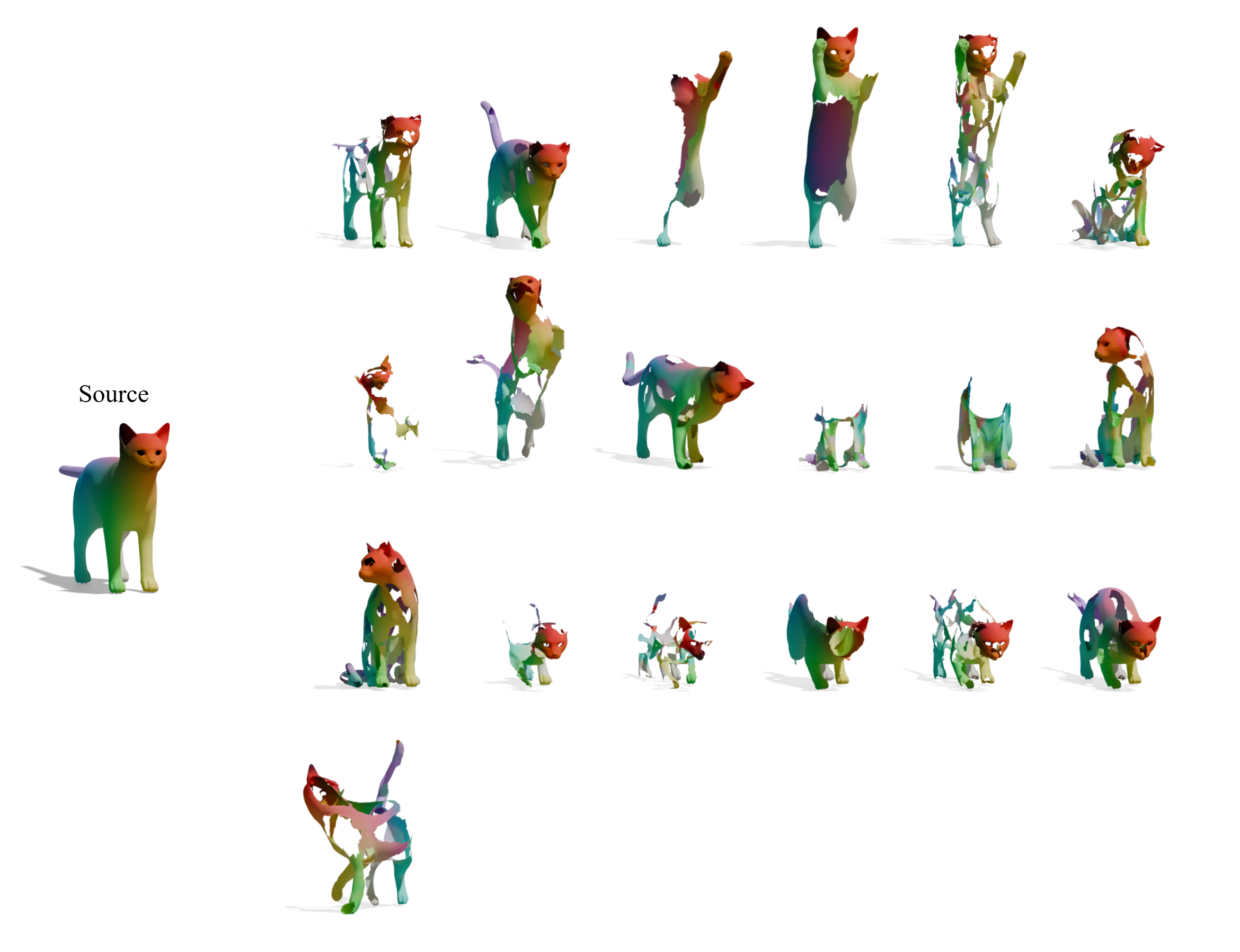}
    \caption{SHREC16-Partiality HOLES cat}
    \label{fig:holes-cat}
\end{figure*}

\begin{figure*}
    \centering
    \hspace{-3em}
    \includegraphics[width=\linewidth]{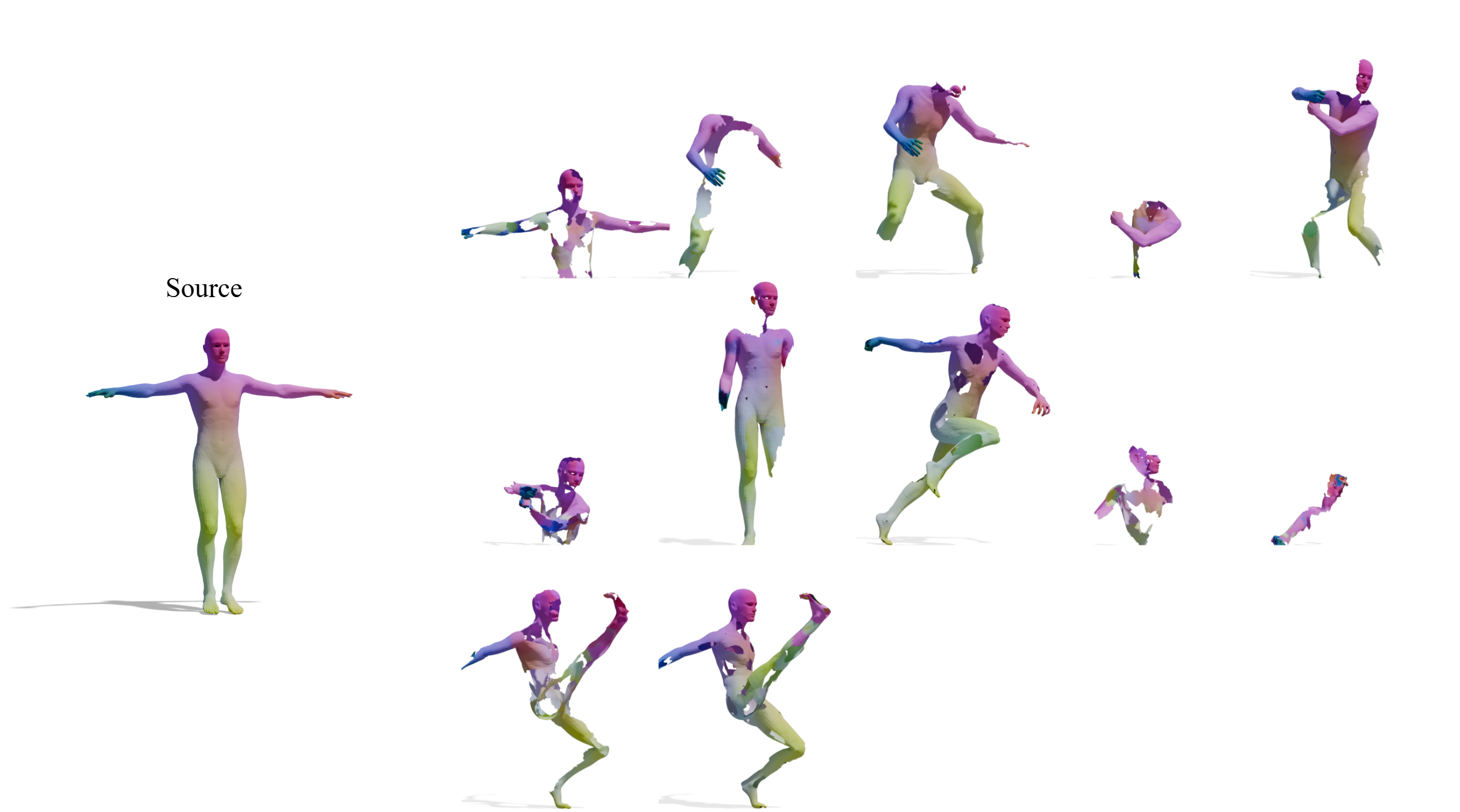}
    \caption{SHREC16-Partiality HOLES david}
    \label{fig:holes-david}
\end{figure*}

\begin{figure*}
    \centering
    \hspace*{-1em}
    \includegraphics[width=0.95\linewidth]{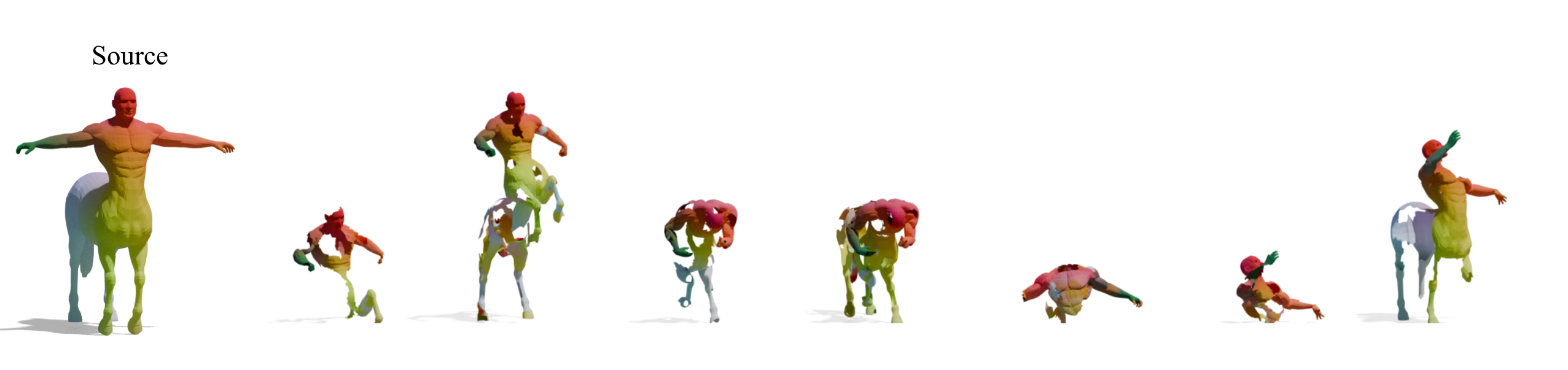}
    \caption{SHREC16-Partiality HOLES centaur}
    \label{fig:holes-centaur}
\end{figure*}

\begin{figure*}
    \centering
    \includegraphics[width=\linewidth]{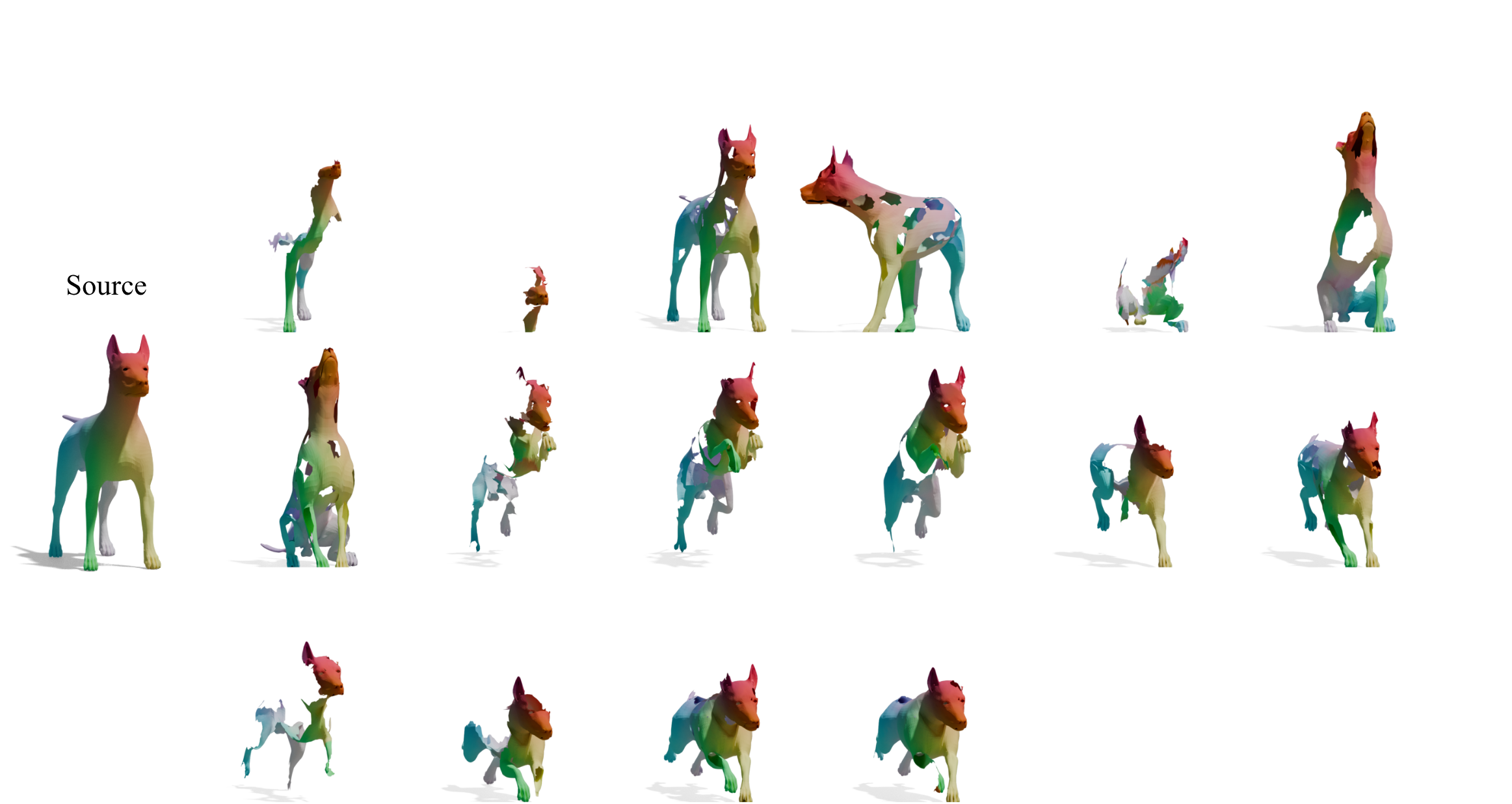}
    \caption{SHREC16-Partiality HOLES dog}
    \label{fig:holes-dog}
\end{figure*}

\begin{figure*}
    \centering
    \includegraphics[width=\linewidth]{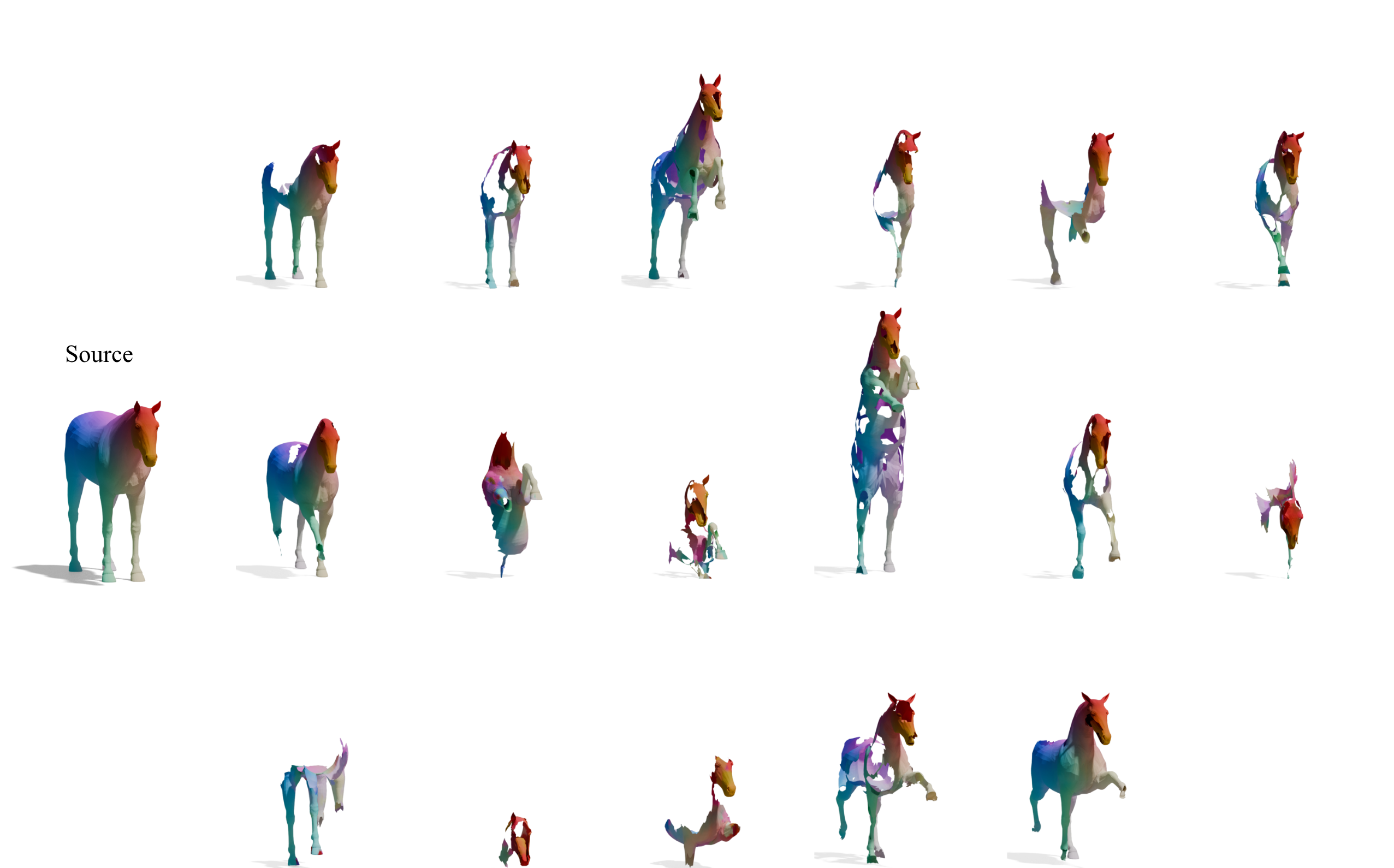}
    \caption{SHREC16-Partiality HOLES horse}
    \label{fig:holes-horse}
\end{figure*}

\begin{figure*}
    \centering
    \hspace{-3em}
    \includegraphics[width=\linewidth]{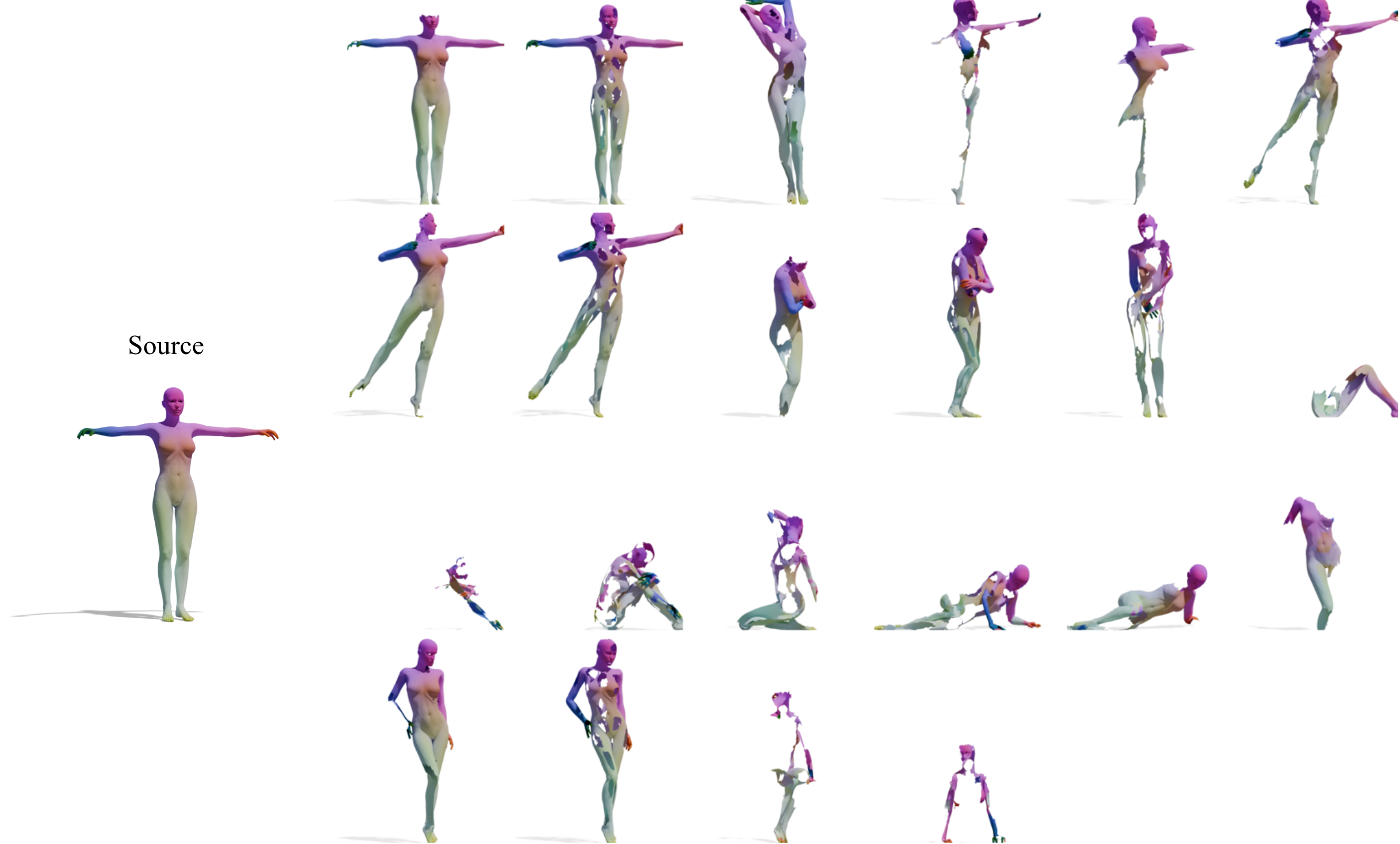}
    \caption{SHREC16-Partiality HOLES victoria}
    \label{fig:holes-victoria}
\end{figure*}

\begin{figure*}
    \centering
    \vspace*{-2em}
    \hspace*{-3em}
    \includegraphics[width=1.1\linewidth]{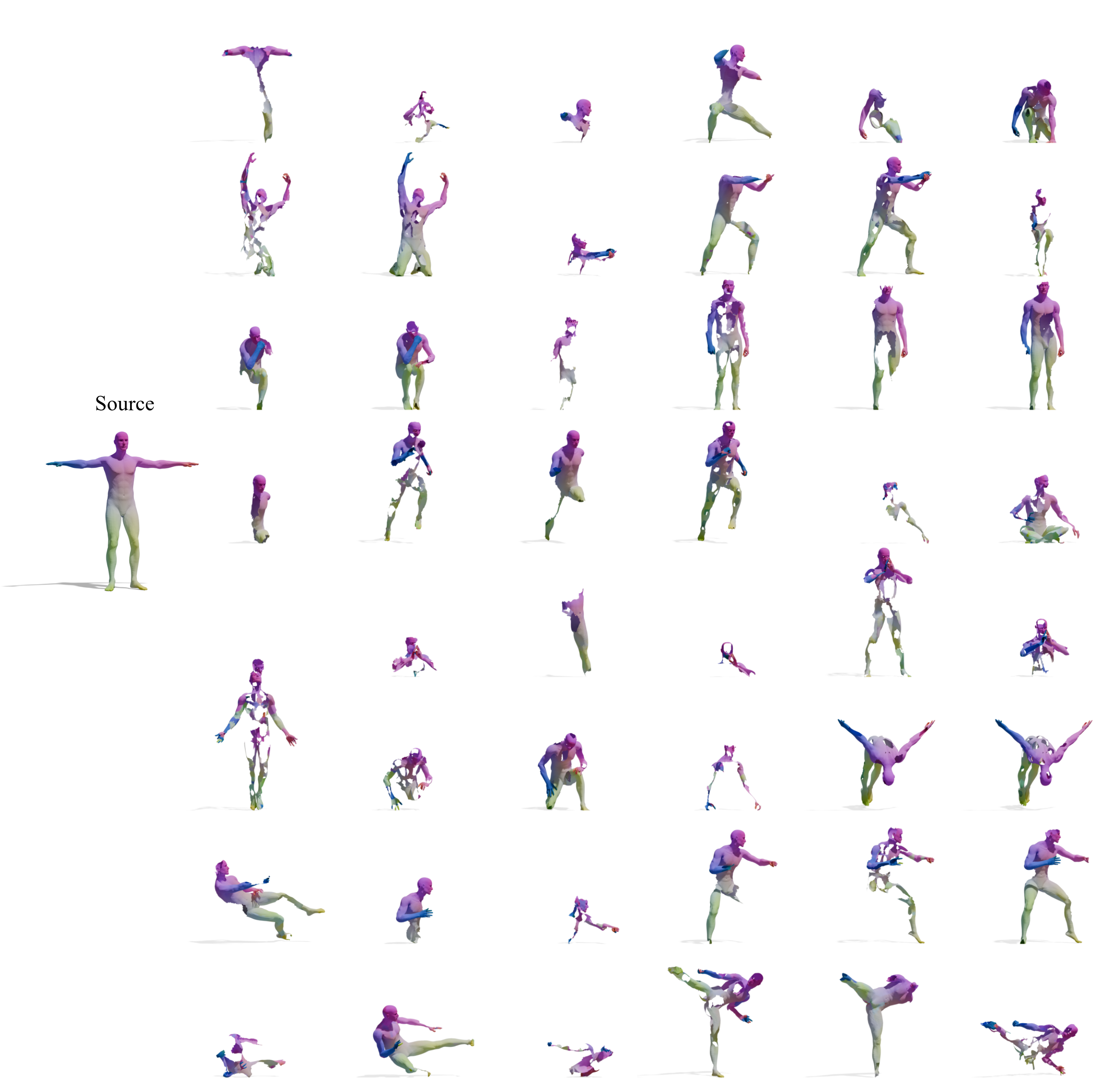}
    \caption{SHREC16-Partiality HOLES michael}
    \label{fig:holes-michael}
\end{figure*}

%% file: figs/wks/boxplot_wks.tex
\definecolor{cPLOT0}{RGB}{28,213,227}
\definecolor{cPLOT1}{RGB}{80,150,80}
\definecolor{cPLOT2}{RGB}{90,130,213}
\definecolor{cPLOT3}{RGB}{247,179,43}
\definecolor{cPLOT5}{RGB}{242,64,0}
\pgfplotsset{%
    label style = {font=\Large},
    tick label style = {font=\large},
    title style =  {font=\Large},
    legend style={  fill= gray!10,
                    fill opacity=0.6, 
                    font=\large,
                    draw=gray!20, %
                    text opacity=1}
}
\begin{tikzpicture}[scale=0.6, transform shape]
    \begin{axis}[
		width=1\columnwidth,
		height=0.7\columnwidth,
		boxplot/draw direction=y,
		grid=major,
		ymin=0,
		ymax=4,
		xmin=0.65,
		xmax=8.35,
	    ytick={0,1,2,3,4},
		yticklabels={0,1, 2, 3, 4,},
		xtick={1, 2, 3, 4, 5, 6, 7, 8, 9},
		xticklabels={$\phantom{1}$, 5, $\phantom{1}$, 10,$\phantom{1}$, 15, $\phantom{1}$, 20, $\phantom{1}$},
		xticklabel style={xshift=-15pt},
        xlabel={Feat. Channel (wks)},
		ylabel={Feat. Mean},
		every boxplot/.style={mark=x,every mark/.append style={mark size=5pt}},
		boxplotcolor/.style={color=#1,solid,very thick,fill=#1!50,mark options={color=#1,fill=#1!70}},
            boxplot/box extend=0.5%
		]
        
    \addplot+[
        boxplotcolor=cPLOT3,
        boxplot prepared={
            lower whisker=0.053429,
            lower quartile=0.4398,
            median=0.78941,
            upper quartile=2.0388,
            upper whisker=3.7972
        },
    ] coordinates {};

    \addplot+[
        boxplotcolor=cPLOT1,
        boxplot prepared={
            lower whisker=0.050434,
            lower quartile=0.12766,
            median=0.15951,
            upper quartile=0.48131,
            upper whisker=1.0108
        },
    ] coordinates {};

    \addplot+[
        boxplotcolor=cPLOT3,
        boxplot prepared={
            lower whisker=0.068584,
            lower quartile=0.44106,
            median=0.86255,
            upper quartile=2.6128,
            upper whisker=2.8763
        },
    ] coordinates {};

    \addplot+[
        boxplotcolor=cPLOT1,
        boxplot prepared={
            lower whisker=0.060452,
            lower quartile=0.12605,
            median=0.22767,
            upper quartile=0.49638,
            upper whisker=1.0501
        },
    ] coordinates {};
    
    \addplot+[
        boxplotcolor=cPLOT3,
        boxplot prepared={
            lower whisker=0.063119,
            lower quartile=0.40389,
            median=0.82369,
            upper quartile=2.8798,
            upper whisker=3.2385
        },
    ] coordinates {};

    \addplot+[
        boxplotcolor=cPLOT1,
        boxplot prepared={
            lower whisker=0.048993,
            lower quartile=0.15735,
            median=0.31463,
            upper quartile=0.79528,
            upper whisker=1.7493
        },
    ] coordinates {};
    
    \addplot+[
        boxplotcolor=cPLOT3,
        boxplot prepared={
            lower whisker=0.058317,
            lower quartile=0.33721,
            median=0.77712,
            upper quartile=2.9412,
            upper whisker=3.5723
        },
    ] coordinates {};

    \addplot+[
        boxplotcolor=cPLOT1,
        boxplot prepared={
            lower whisker=0.094227,
            lower quartile=0.19173,
            median=0.39379,
            upper quartile=0.89093,
            upper whisker=1.9358
        },
    ] coordinates {};

	\end{axis}
\end{tikzpicture}

%% file: figs/wks/boxplot_xyz.tex
\definecolor{cPLOT0}{RGB}{28,213,227}
\definecolor{cPLOT1}{RGB}{80,150,80}
\definecolor{cPLOT2}{RGB}{90,130,213}
\definecolor{cPLOT3}{RGB}{247,179,43}
\definecolor{cPLOT5}{RGB}{242,64,0}
\pgfplotsset{%
    label style = {font=\Large},
    tick label style = {font=\large},
    title style =  {font=\Large},
    legend style={  fill= gray!10,
                    fill opacity=0.6, 
                    font=\large,
                    draw=gray!20, %
                    text opacity=1}
}
\begin{tikzpicture}[scale=0.6, transform shape]
    \begin{axis}[
		width=1\columnwidth,
		height=0.7\columnwidth,
		boxplot/draw direction=y,
		grid=major,
		ymin=-0.7,
		ymax=0.6,
		xmin=0.65,
		xmax=6.35,
	    ytick={-0.6, -0.3, 0, 0.3, 0.6},
		yticklabels={-0.6, -0.3, 0, 0.3, 0.6},
		xtick={1, 2, 3, 4, 5, 6, 7, 8, 9},
		xticklabels={$\phantom{1}$, X, $\phantom{1}$, Y,$\phantom{1}$, Z, $\phantom{1}$, 20, $\phantom{1}$},
		xticklabel style={xshift=-15pt},
        xlabel={Feat. Channel (xyz)},
		every boxplot/.style={mark=x,every mark/.append style={mark size=5pt}},
		boxplotcolor/.style={color=#1,solid,very thick,fill=#1!50,mark options={color=#1,fill=#1!70}},
            boxplot/box extend=0.5%
		]
        
    \addplot+[
        boxplotcolor=cPLOT3,
        boxplot prepared={
            lower whisker=-0.16713,
            lower quartile=-0.081806,
            median=0.015068,
            upper quartile=0.072461,
            upper whisker=0.098933
        },
    ] coordinates {};

    \addplot+[
        boxplotcolor=cPLOT1,
        boxplot prepared={
            lower whisker=-0.057712,
            lower quartile=-0.016124,
            median=0.00046633,
            upper quartile=0.021131,
            upper whisker=0.053947
        },
    ] coordinates {};

    \addplot+[
        boxplotcolor=cPLOT3,
        boxplot prepared={
            lower whisker=-0.24452,
            lower quartile=-0.089113,
            median=0.095405,
            upper quartile=0.26477,
            upper whisker=0.29558
        },
    ] coordinates {};

    \addplot+[
        boxplotcolor=cPLOT1,
        boxplot prepared={
            lower whisker=-0.61832,
            lower quartile=-0.21584,
            median=0.044799,
            upper quartile=0.28689,
            upper whisker=0.54048
        },
    ] coordinates {};
    
    \addplot+[
        boxplotcolor=cPLOT3,
        boxplot prepared={
            lower whisker=-0.041645,
            lower quartile=-0.012556,
            median=-0.0021148,
            upper quartile=0.008481,
            upper whisker=0.03694
        },
    ] coordinates {};

    \addplot+[
        boxplotcolor=cPLOT1,
        boxplot prepared={
            lower whisker=-0.050331,
            lower quartile=-0.022025,
            median=-0.0043905,
            upper quartile=0.009749,
            upper whisker=0.053343
        },
    ] coordinates {};
    
	\end{axis}
\end{tikzpicture}